\documentclass{article}

\usepackage{PRIMEarxiv}

\usepackage[utf8]{inputenc} 
\usepackage[T1]{fontenc}    
\usepackage{hyperref}       
\usepackage{url}            
\usepackage{booktabs}       
\usepackage{nicefrac}       
\usepackage{microtype}      
\usepackage{lipsum}
\usepackage{fancyhdr}       
\usepackage{graphicx}       
\graphicspath{{media/}}     

\usepackage{amsthm}
\usepackage{amsmath,amssymb,amsfonts,mathtools,bm}
\usepackage{algorithm}
\usepackage{algorithmic}
\usepackage{subcaption}
\usepackage{hyperref}
\hypersetup{hidelinks}

\usepackage[noabbrev]{cleveref}

\DeclareMathOperator*{\argmax}{arg\,max}

\newtheorem{theorem}{Theorem}
\newtheorem{lemma}{Lemma}
\newtheorem{claim}{Claim}

\newtheorem{assumption}{Assumption}

\newenvironment{claimproof}{\noindent\emph{Proof of claim.}}{\hfill$\qed$}

\crefname{assumption}{assumption}{assumptions}

\newcommand{\vect}[1]{\ensuremath{\bm{#1}}}
\newcommand{\E}{\ensuremath{\mathbb{E}}}
\newcommand{\Var}{\ensuremath{Var}}
\newcommand{\norm}[1]{\left\lVert#1\right\rVert}
\newcommand{\parenthese}[1]{\left(#1\right)}
\newcommand{\bracket}[1]{\left[#1\right]}

\newcommand{\scalarproduct}[1]{\left\langle#1\right\rangle}
\newcommand{\curlybracket}[1]{\left\{#1\right\}}

\title{One Gradient Frank-Wolfe for Decentralized Online Convex and Submodular Optimization
\thanks{Supported by the Multidisciplinary Institute in Artificial Intelligence, Univ.Grenoble Alpes, France (ANR-19-P3IA-0003) } 
}

\author{
  Tuan-Anh Nguyen \\
  LIG, Inria \\ 
  University Grenoble-Alpes \\
  Saint-Martin-d'Hères, France \\
  \texttt{tuan-anh.nguyen@inria.fr}
  \And 
  Nguyen Kim Thang \\
  LIG, Inria, CNRS, Grenoble INP \\
  University Grenoble-Alpes \\
  Saint-Martin-d'Hères, France\\
  \texttt{kim-thang.nguyen@univ-grenoble-alpes.fr} \\
   \And
  Denis Trystram \\
  LIG, Inria, CNRS, Grenoble INP \\
  University Grenoble-Alpes \\
  Saint-Martin-d'Hères, France \\
  \texttt{denis.trystram@imag.fr} \\
}

\begin{document}
\maketitle

\begin{abstract}
Decentralized learning has been studied intensively in recent years motivated by its wide applications in the context of federated learning. The majority of previous research focuses on the offline setting in which the objective function is static. However, the offline setting becomes unrealistic in numerous machine learning applications that witness the change of massive data. In this paper, we propose \emph{decentralized online} algorithm for convex and continuous DR-submodular optimization, two classes of functions that are present in a variety of machine learning problems. Our algorithms achieve performance guarantees comparable to those in the centralized offline setting. Moreover, on average, each participant performs only a \emph{single} gradient computation per time step. Subsequently, we extend our algorithms to the bandit setting. Finally, we illustrate the competitive performance of our algorithms in real-world experiments.
\end{abstract}



\section{Introduction}
\label{chap:introduction}

Learning over data generated by sensors and mobile devices has gained a high interest in recent years due to the continual interaction with users and the environment on a timely basis. The patterns related to user's behavior, preference, and the surrounding stochastic events become a promising source for machine learning applications to be more and more reliable. However, collecting such data in a centralized location has become problematic due to privacy concerns and the high cost of data transfer over the network. Consequently, the learning methods that can leave the data locally while efficiently exploiting data patterns, such as decentralized learning, are emerging as an alternative to traditional centralized learning.

Under the optimization scheme, learning in a decentralized manner consists of multiple interconnected agents cooperating to optimize a global objective function where each agent detains partial information of the interested function. Several works \cite{Deori:2016, Reisizadeh:2019, Yuan:2016, Duchi:2012, Zheng:2018}  have considered this setting for convex and strongly convex functions. \cite{WaiLafond17:Decentralized-Frank--Wolfe} also study the problem when the objective function is generally non-convex whereas \cite{MokhtariHK18, xie19b} proposes a decentralized algorithm to maximize monotone submodular functions for both continuous and discrete domains. However, these works only consider the offline setting which is not realistic since data constantly evolve in many real-world applications. In this paper, we study decentralized online algorithms for optimizing both convex and submodular functions. 

\paragraph{Problem definition.} 
Formally, we are given a convex set $\mathcal{K} \subseteq \mathbb{R}^d$ (w.l.o.g one can assume that $\mathcal{K} \subseteq [0,1]^d$) 
and a set of agents connected over a network represented by a graph $G = (V, E)$ 
where $n  = |V|$ is the number of agents. 
At every time $1 \leq t \leq T$, each agent $i \in V$ can communicate with (and only with) its immediate neighbors, i.e., adjacent agents in $G$ and 
takes a decision $\vect{x}^{i}_{t} \in \mathcal{K}$. 
Subsequently, a cost/reward function $f^{i}_{t}: \mathcal{K} \rightarrow \mathbb{R}$
is revealed adversarially and locally to agent $i$. Note that in the \emph{bandit} setting, agent $i$ observes only the value $f^{i}_{t}(\vect{x}^{i}_{t})$ 
instead of the whole function $f^{i}_{t}$. Although each agent $i$ observes only function $f^{i}_{t}$ (or the value $f^{i}_{t}(\vect{x}^{i}_{t})$ in the bandit setting), 
agent $i$ is interested in the cumulating cost/reward $F_{t}(\cdot)$  where 
$F_{t}(\cdot) := \frac{1}{n} \sum_{j=1}^{n} f^{j}_{t}(\cdot)$. In particular, at time $t$, 
the cost/reward of agent $i$ with the its chosen $\vect{x}^{i}_{t}$ is $F_{t}(\vect{x}^{i}_{t})$.

In the context of convex minimization, the functions $f_{t}^{i}$'s are convex and 
the goal of each agent $i$ is to minimize the total cumulating cost $\sum_{t=1}^{T} F_{t}(\vect{x}^{i}_{t})$
via local communication with its immediate neighbors.  
Our objective is to design an algorithm with small regret. 
An online algorithm is \emph{$\mathcal{R}_T$-regret} if for every agent $1 \leq i \leq n$, 
\begin{align*}
\sum_{t = 1}^T F_t(\vect{x}_t^i) - \min_{\vect{x} \in \mathcal{K}} \sum_{t=1}^T F_t(\vect{x}) \leq \mathcal{R}_T
\end{align*}

In the context of monotone DR-submodular maximization, the functions $f_{t}^{i}$'s are monotone DR-submodular. 
Roughly speaking, a bounded differentiable and non-negative function $F:[0,1]^d \rightarrow \mathbb{R}_+$ 
is \emph{DR-submodular} if for every $\vect{x}, \vect{y} \in [0,1]^d$ satisfying $x_{i} \leq y_i, \forall i \in [d]$, 
we have $\nabla F(\vect{x}) \geq \nabla F(\vect{y})$. 
The goal of each agent $i$ is to maximize the total cumulating reward $\sum_{t=1}^{T} F_{t}(\vect{x}_{t}^{i})$, again
via local communication with its immediate neighbors.  
Our objective is to design an algorithm with an approximation ratio as close to 1 as possible and together with a small regret. 
An online algorithm has a \emph{$\rho$-regret} of $\mathcal{R}_T$ if for every agent $1 \leq i \leq n$, 
\begin{align*}
\rho \cdot \max_{\vect{x} \in \mathcal{K}}  \sum_{t=1}^T F_t(\vect{x})  - \sum_{t = 1}^T F_t(\vect{x}_t^i) \leq \mathcal{R}_T
\end{align*}

\subsection{Our contribution}
The challenge in designing robust and efficient algorithms for these problems is to simultaneously address the following issues:
\begin{itemize}
    \item Uncertainty (online setting, agents observe their loss functions only after selecting their decisions).
    \item Partial information (decentralized setting, agents know only its local loss functions while attempting to minimize the cumulated cost).
    \item Low computation/communication resources of agents (so it is desirable that each agent performs a small number of gradient computations and communications). 
    \item Additionally, in the bandit setting, one has only limited feedback (agents can only observe the function value of their decisions).
\end{itemize}

We present performance-guaranteed algorithms for solving the constraint convex and continuous DR-submodular optimization problem in the decentralized and online setting with \emph{only} one gradient evaluation and low communications per agent per time step on average. Specifically, 
our algorithms achieve the regret and the $\parenthese{1-\frac{1}{e}}$-regret bounds of $O(T^{4/5})$ for both convex and monotone continuous DR-submodular functions. 
Using a one-point gradient estimator \cite{FlaxmanKalai05:Online-convex}, we extend the algorithms to the bandit setting in which the gradient is unavailable to the agents. 
We obtain the $\parenthese{1-\frac{1}{e}}$-regret bound of $O(T^{8/9})$ for the bandit setting. 
It should be noted that the $\parenthese{1-\frac{1}{e}}$-regret of $O(T^{4/5})$ and $O(T^{8/9})$ matches the regret guarantees in the centralized online settings.
Besides, one can convert the algorithm to be projection-free (by selecting suitable oracles). 
This property allows the algorithm to be implemented in various contexts based on the computing capacity of local devices. 
We demonstrate the practical application of our algorithm on a Movie Recommendation problem and present a thorough analysis of different aspects of the performance guarantee, the effects of network topology, and decentralization, which are predictably explained by our theoretical results.

\begin{table}[ht!]
\centering
\begin{tabular*}{\textwidth}{c @{\extracolsep{\fill}} ccccc}
\toprule
Algorithm & Stochastic & $(1-1/e)$-Regret & Communications & Gradient \\ 
 & Gradient &  &  & Evaluations \\ 
\midrule
DMFW  & Yes & $\mathcal{O}(T^{1/2})$  &  $2 \cdot T^{3/2}$ & $T^{3/2}$ \\ 
\textbf{Monode-FW}  & Yes & $\mathcal{O}(T^{4/5})$ & $2 \cdot$ \#neighbors & 1\\ 
\textbf{Bandit Monode-FW}  & - & $\mathcal{O}(T^{8/9})$ & $2 \cdot$ \#neighbors & - \\
\bottomrule
\end{tabular*}
\caption{Comparison of previous work on \emph{adversarial} decentralized online monotone DR-submodular maximization (DMFW \cite{submod-timevarying}) and our proposed algorithms (in bold). The communications and gradient evaluations are mesured per agent per time step.}
\hspace{0.8cm}
\end{table}
\subsection{Related Works}
\label{chap:relatedworks}

\paragraph{Decentralized Online Optimization.}  %
\cite{Yan:2013} introduced decentralized online projected subgradient descent and showed vanishing regret for convex and strongly convex functions. 
In contrast,~\cite{Hosseini:2013} extended distributed dual averaging technique to the online setting, using a general regularized projection for both unconstrained and constrained optimization.
A distributed variant of online conditional gradient~\cite{Hazanothers16:Introduction-to-online} was designed and analyzed in~\cite{Zhang:2017} that requires linear minimizers and uses exact gradients.
However, computing exact gradients may be prohibitively expensive for moderately sized data and intractable when a closed-form does not exist. \cite{submod-timevarying} proposes a decentralized online algorithm for maximizing monotone submodular function on a time-varying network using stochastic gradient estimate and multiple optimization oracles. This work achieves the optimal regret bound of $O(T^{1/2})$ but requires $T^{3/2}$-gradient evaluation and communication per function. In this work, we advance further in designing a distributed algorithm that uses stochastic gradient estimates and requires only one gradient evaluation.

\paragraph{Monotone DR-submodular Maximization.}
The maximization of monotone DR-sub\-modular functions has been investigated in both offline and online settings. For the offline case, \cite{BianMirzasoleiman17:Guaranteed-Non-convex} examined the problem where the constraint set is a down-closed convex set and demonstrated that the greedy method \cite{CalinescuChekuri11:Maximizing-a-monotone}, a variation of the Frank-Wolfe algorithm, ensures a $(1-1/e)$-approximation. \cite{HassaniSoltanolkotabi17:Gradient-methods} demonstrated the restriction of the greedy method in a stochastic environment where only unbiased gradient estimates are available. Later, \cite{MokhtariHassani18:Conditional-Gradient} introduced an algorithm for maximizing monotone DR-submodular function over the general convex set using new variance reduction techniques to accomplish $(1-1/e)$-approximation in a stochastic setting. \cite{ChenHarshaw18:Projection-Free-Online} suggested a method that achieves $(1-1/e, O(\sqrt{T}))$-regret for maximizing monotone Dr-submodular over a general convex set in an online setting. Subsequently, \cite{Zhang:2019} introduced an approach that reduces the number of per-function gradient evaluations from $T^{3/2}$ to 1, while maintaining the same approximation ratio of $(1-1/e)$. They also presented a bandit approach that achieves an expected $(1-1/e)$-approximation ratio with regret $T^{8/9}$ to tackle the same problem. 



\section{Preliminaries and Notations}
\label{chap:preliminaries}

\label{sec:math_notation}
%
We begin by explaining the notations and concepts that will be used throughout the paper. Given an undirected graph $G = (V, E)$, the set of neighbors of an agent $i \in V$ is denoted $N(i) := \{j \in V: (i,j) \in E\}$. Consider the following symmetric matrix $\mathbf{W} \in \mathbb{R}_{+}^{n \times n}$ with entry $W_{ij}$ defined as follows. 
\begin{align*}
    W_{ij} = \begin{cases}
            \dfrac{1}{1+\max\{d_i, d_j\}} & \text{if $(i,j) \in E$}\\
            0 &  \text{if $(i,j) \not\in E$,$i \neq j$}\\
            1 - \sum_{j \in N(i)} W_{ij} & \text{if $i=j$}
        \end{cases}
\end{align*}
where $d_i = |N(i)|$, the degree of vertex $i$.
In fact, the matrix $\mathbf{W}$ is doubly stochastic, i.e $\mathbf{W} \vect{1} = \mathbf{W}^T \vect{1} = \vect{1}$ and so it inherits several useful properties of doubly stochastic matrices. 
We use boldface letter e.g $\vect{x}$ to represent vectors. We denote by $\vect{x}^i_{q,k}$ the decision vector of agent $i$ at time step $k$ of phase $q$. If not specified otherwise, we suppose that the constraint set $\mathcal{K}$ is a bounded convex set with diameters $D = \sup_{\vect{x}, \vect{y} \in \mathcal{K}} \| \vect{x} - \vect{y} \|$ and radius $R = \sup_{\vect{x} \in \mathcal{K}} \norm{\vect{x}}$. For two vectors $\vect{x}, \vect{y} \in \mathbb{R}^d$, we note $\vect{x} \leq \vect{y}$ if $x_i \leq y_i ~\forall i$. We note $\mathbb{B}_d$ and $\mathbb{S}_d$ the $d$-dimensional unit ball and the unit sphere, respectively. 


A continuous function $F : [0,1]^d \rightarrow \mathbb{R}_+$ is \emph{DR-submodular} if for any vectors $\vect{x}$, $\vect{y} \in [0,1]^d$ such that $\vect{x} \leq \vect{y}$, for a constant $\alpha > 0$ and any basis vectors $\vect{e}_i = \parenthese{0, \dots, 1, \dots, 0}$ such that $\vect{x} + \alpha \vect{e}_i \in [0,1]^d$ and $\vect{y} + \alpha \vect{e}_i \in [0,1]^d$. 
\begin{align}
    F(\vect{x} + \alpha \vect{e}_i) - F(\vect{x}) \geq F(\vect{y} + \alpha \vect{e}_i) - F(\vect{y})
\end{align}
For a differentiable function, the DR-property is equivalent to $\nabla F(\vect{x}) \geq \nabla F(\vect{y})$, $\forall \vect{x} \leq \vect{y} \in [0,1]^d$. More over, if $F$ is twice-differentiable, the DR-property is equivalent to all entries of the Hessian matrix being non-positive i.e $\forall 1 \leq i,j \leq d$, $\frac{\partial^2 F}{\partial \vect{x}_i \partial \vect{x}_j } \leq 0$. A function $F$ is \emph{monotone} if $\forall \vect{x} \leq \vect{y} \in [0,1]^d$, we have $F(\vect{x}) \leq F(\vect{y})$.


A function $F$ is \emph{$\beta$-smooth} if for all $\vect{x}, \vect{y} \in \mathcal{K}$ :
\begin{align*}
     F(\vect{y}) \leq F(\vect{x}) + \langle \nabla F(\vect{x}), \vect{y}-\vect{x} \rangle + \frac{\beta}{2}\|\vect{y}-\vect{x}\|^2 
\end{align*} or equivalently $\| \nabla F(\vect{x}) - \nabla F(\vect{y})\| \leq \beta \| \vect{x} - \vect{y} \|$. Also, we say a function $F$ is \emph{$G$-Lipschitz} if for all $\vect{x}, \vect{y} \in \mathcal{K}$
\begin{align*}
    \|F(\vect{x}) - F(\vect{y}) \| \leq G \|\vect{x} - \vect{y}\|
\end{align*}

In this paper, we employ linear optimization oracles in our algorithm to solve an online linear optimization problem given a feedback function and a constraint set. In particular, in the online linear optimization problem, one must choose $\vect{u}^{t} \in \mathcal{K}$ at every time $1 \leq t \leq T$. The adversary then discloses a vector $\vect{d}^{t}$ and feedbacks the cost function $\langle \cdot , \vect{d}^t \rangle$ where the goal is to minimize the regret of the linear objective.
Several algorithms \cite{Hazanothers16:Introduction-to-online}, including the projection-free follow-the-perturbed-leader algorithm offer an optimum regret bound of $\mathcal{R}_T = O(\sqrt{T})$ for the online linear optimization problem. One of these methods can be used as an oracle to solve the online linear optimization problem.

In practice, it may not be possible to use a full gradient due to the vast quantity of data and processing restrictions. To address this issue, our approach utilizes an unbiased stochastic gradient in place of the gradient and proposes a variance reduction technique for distributed optimization based on a rigorous analysis that may be applied to problems of independent interest. We make the following assumptions for the next two sections.

\begin{assumption}
\label[assumption]{assum:assum_1}
We let $k_0$ to be the smallest integer such that 
$\lambda_2(\mathbf{W}) \leq  \parenthese{\frac{k_0}{k_0 +1}}^2$.
\end{assumption}

\begin{assumption}
\label[assumption]{assum:stoch_grad}
The function $f_t$ is $G$-Lipschitz and $\beta$-smooth. Its stochastic gradient $\widetilde{\nabla} f_{t} (\vect{x})$ is unbiased, uniformly upper-bounded and has a bounded variance, i.e.,  $\E \bracket{\widetilde{\nabla} f_{t} (\vect{x})} = \nabla f_{t} (\vect{x})$, $\norm{\widetilde{\nabla} f_t(\vect{x})} \leq G_0$, and $\E \bracket{\norm{\widetilde{\nabla} f_{t} (\vect{x}) - \nabla f_{t}(\vect{x})}^2} \leq \sigma^2_0$.

\end{assumption}

\begin{assumption}
\label[assumption]{assum:max_bound_f}
For all $t \in \bracket{T}$ and  $i \in \bracket{n}$, $\sup_{\vect{x} \in \mathcal{K}} | f_t^i (\vect{x}) | \leq B$
\end{assumption}

\begin{assumption}
\label[assumption]{assum:ball_r}
There exist a number $r \geq 0$ such that $r \mathbb{B}_d \subseteq \mathcal{K}$
\end{assumption}

\section{Full Information Setting}
\label[section]{chap:formulation}


This section thoroughly describes the algorithm for both convex and DR-submodular optimization. Recall that each agent receives a function $f_t^i$ at every time $t \in \bracket{T}$. We partition time steps into $Q$ blocks, each of size $K$ so that $T = QK$. For each block $q \in \bracket{Q}$, we define $f^i_q$ as the average of the $K$ functions within the block. Additionally, each agent $1 \leq i \leq n$ maintains $K$ online linear optimization oracles $\mathcal{O}_{i,1}, \ldots, \mathcal{O}_{i,K}$. Let $\sigma_q \in \mathfrak{S}_K$ be a random permutation of function indexes for all agents. 

At a high level, at each block $q$, the agent $i$ performs $K$-steps of Frank-Wolfe algorithm, where the update vector is a combination of the oracles' outputs and the aggregate of its neighbors' current decisions. The final decision $\vect{x}^i_q$ for the block $q$ is disclosed at the end of $K$ steps, such that at each time step in the block, agent $i$ plays the same decision $\vect{x}^i_q$.

More specifically, following the Frank-Wolfe steps, agent $i$ performs $K$ gradient updates using the estimators $f^i_{\sigma_q(k)}$. It calculates the stochastic gradient of the permuted function $f^i_{\sigma_q (k)}$ evaluated at the corresponding decision vector $\vect{x}^i_{q,k}$ and thereafter exchanges information with its neighbors. It then computes a variance reduction version $\widetilde{\vect{a}}^{i}_{q,k}$ of the vector $\widetilde{\vect{d}}^{i}_{q,k}$ and returns $\langle \widetilde{\vect{a}}^i_{q,k}, \cdot \rangle$ as the cost function at time $\sigma^{-1}_q (k)$ to the oracle $\mathcal{O}^i_{k}$. The vectors $\widetilde{\vect{d}}^i_{q,k}$ are subtly constructed to capture progressively more information on the accumulating cost functions. 

Note that the use of random permutation $\sigma_q$ is crucial here. By that, all the permuted functions $f^i_{\sigma_q(k)}$ become an estimation of $f^i_q$, i.e., 
$\E[f^i_{\sigma_q(k)}] = f^i_q$. Therefore the gradient of $f^i_{\sigma_q(k)}$ is likewise an estimation of the gradient of $f^i_q$. One can think of $f^i_q$ as an artificial objective function for which we have access to its gradient estimates, where each estimation is one gradient evaluation per function within the block. As a result, conducting $K$ gradient updates of $f^i_q$ turns out to be executing one gradient update for each of the $K$ functions. Using this approach, initiated in \cite{Zhang:2019}, we can effectively reduce the gradient evaluation number to 1 for each arriving function $f^i_t$.

Since we deal with both convex and submodular, there are modifications to adapt for both kinds of optimization problem. The online optimization oracle's objective function should be minimized for convex optimization and maximized for submodular optimization. The decision update for convex problems is a convex combination of the aggregated neighbors' decisions $\vect{y}^i_{q,k}$ and the oracle's output $\vect{v}^i_{q,k}$, i.e.,
\begin{align}
\label{eq:covex-update}
    \vect{x}^i_{q,k+1} = (1-\eta_k)\vect{y}^i_{q,k} + \eta_k \vect{v}^i_{q,k}, \quad \eta_k \in \bracket{0,1} 
\end{align}
whereas the update for the submodular optimization problem is achieved by shifting the aggregated decisions towards the direction of the oracle's output by a step-size $\eta_k$, i.e.,
\begin{align}
\label{eq:submodular-update}
    \vect{x}^i_{q,k+1} = \vect{y}^i_{q,k} + \eta_k \vect{v}^i_{q,k}, \quad \eta_k \in \bracket{0,1} 
\end{align}
For convex functions, the initialization can be any random point inside the constraint set, however for submodular functions, this value should be set to 0. 

The formal description is given in Algorithm~\ref{algo:online-dist-FW}. The proof of the following lemmas and theorems can be found in \Cref{chap:analysis}.

\begin{algorithm}[ht!]
\begin{flushleft}
\textbf{Input}:  A convex set $\mathcal{K}$, 
	a time horizon $T$, a block size $K$, online linear optimization oracles $\mathcal{O}_{i,1}, \ldots, \mathcal{O}_{i,K}$ for each agent $1 \leq i \leq n$, 
	step sizes $\eta_k \in (0, 1)$ for all $1 \leq k \leq K$, number of blocks $Q=T/K$
\end{flushleft}
\begin{algorithmic}[1]
\STATE Initialize linear optimizing oracle $\mathcal{O}^i_k$ for all $1 \leq k \leq K$
\FOR {$q = 1$ to $Q$}	 		
	\FOR{every agent $1 \leq i \leq n$}	%
		\STATE Initialize $\vect{x}^i_{q,1}$ and set $\widetilde{\vect{{a}}}^t_{i,0} \gets 0$ 
		\FOR{$1 \leq k \leq K$}
			\STATE Let $\vect{v}^{i}_{q,k}$ be the output of oracle $\mathcal{O}^i_{k}$ at phase $q$. \label{online-oracle}
			\STATE Send $\vect{x}^{i}_{q,k}$ to all neighbours $N(i)$
			\STATE \label{alg:y} 
				Once receiving $\vect{x}^{j}_{q,k}$ from all neighbours $j \in N(i)$, 
				set $\vect{y}^{i}_{q,k} \gets \sum_{j} W_{ij} \vect{x}^{j}_{q,k}$.
			\STATE \label{alg:x} Update $\vect{x}^i_{q,k+1}$ as (\ref{eq:covex-update}) or (\ref{eq:submodular-update}). \label{update-x}
		\ENDFOR
		\STATE Choose $\vect{x}^{i}_{q} \gets \vect{x}^{i}_{q,K+1}$ and agent $i$ plays the same $\vect{x}^i_{q}$ for every time $t$ in phase $q$.
		\STATE Let $\sigma_q$ be a random permutation of $1, \ldots, K$ --- times in phase $q$.
		\FOR{$1 \leq k \leq K$}
		    \STATE Let $s = \sigma_q^{-1}(k)$ 
		    \STATE Query the values of 
		    $\widetilde{\nabla} f^{i}_{k}(\vect{x}^i_{q,s})$ 
		\ENDFOR
		\STATE Set $\widetilde{\vect{g}}^{i}_{q,1} \gets \widetilde{\nabla} f^{i}_{\sigma_q (1)}(\vect{x}^{i}_{q,1})$
			\FOR{$1 \leq k \leq K$}
				\STATE Send $\widetilde{\vect{g}}^{i}_{q,k}$ to all neighbours $N(i)$.
				\STATE After receiving $\widetilde{\vect{g}}^{j}_{q,k}$ from all neighbours $j \in N(i)$, compute
					$\widetilde{\vect{d}}^{i}_{q,k} \gets  \sum_{j \in N(i)} W_{ij} \widetilde{\vect{g}}^{j}_{q,k}$
					and
					$\widetilde{\vect{g}}^{i}_{q,k + 1} \gets \bigl( \widetilde{\nabla} f^{i}_{\sigma_q (k+1)}(\vect{x}^i_{q,k+1}) 
						-  \widetilde{\nabla} f^{i}_{\sigma_q(k)}(\vect{x}^{i}_{q,k}) \bigr) + \widetilde{\vect{d}}^{i}_{q,k}$ \label{alg:d-update}
				\STATE $\widetilde{\vect{a}}^i_{q,k} \gets (1 - \rho_k) \cdot \widetilde{\vect{a}}^i_{q, k-1} + \rho_k \cdot \widetilde{\vect{d}}^{i}_{q, k}$.
				\STATE Feedback function $\langle \widetilde{\vect{a}}^{i}_{q,k}, \cdot \rangle$ 
				to oracles $\mathcal{O}^i_k$. (The cost of the oracle $\mathcal{O}^i_k$ at block $q$ is 
				$\langle \widetilde{\vect{a}}^{i}_{q,k}, \vect{v}^{i}_{q,k}  \rangle$.)
			\ENDFOR
	\ENDFOR
\ENDFOR
\end{algorithmic}
\caption{Monode Frank-Wolfe}
\label{algo:online-dist-FW}
\end{algorithm}


\begin{lemma}
\label[lemma]{lmm:bound_d}
Let $V_d = 2nG \parenthese{\frac{\lambda_2}{1-\lambda_2}+1}$, the local gradient is uniformly upper-bounded, i.e, $\forall i \in \bracket{n}, \forall k \in \bracket{K}$.  $\norm{\vect{d}^i_{q,k}} \leq V_d$.
\end{lemma}

\begin{lemma}
\label[lemma]{lmm:stoch_variance}
Under \Cref{assum:stoch_grad} and let $\sigma_1^2 = 4n \bracket{\parenthese{\frac{G+G_0}{\frac{1}{\lambda_2}-1}}^2 + 2\sigma_0^2}$. For $i \in \bracket{n}, k \in \bracket{K}$, the variance of the local stochastic gradient is uniformly bounded i.e 
$    \E \bracket{ \norm{
        \vect{d}^i_{q,k} - \widetilde{\vect{d}}^i_{q,k}
    }^2} \leq\sigma_1^2$
\end{lemma}


\begin{theorem}
\label[theorem]{thm:convex}
Given a convex set $\mathcal{K}$ and assume that $F_t$ is a convex function.
Setting $Q=T^{2/5}, K=T^{3/5}, T=QK$ and step-size $\eta_k = \frac{1}{k}$. Let $\rho_k = \frac{2}{\parenthese{k+3}^{2/3}}$ and $\rho_k = \frac{1.5}{\parenthese{K-k+2}^{2/3}}$ when $k \in \bracket{1,\frac{K}{2}}$ and $k \in \bracket{\frac{K}{2} +1,K}$ respectively. Then, the expected regret of \Cref{algo:online-dist-FW} is at most
\begin{align}
    \E \bracket{\mathcal{R}_T}
    \leq \parenthese{GD + 2\beta D^2}T^{2/5} 
        + \parenthese{C + 6D\parenthese{N + \sqrt{M}}}T^{4/5} 
        + \frac{3}{5}\beta D^2 T^{2/5}\log (T)
\end{align}
where $N = k_0 \cdot nG\max \{\lambda_2 \parenthese{1 + \frac{2}{1-\lambda_2}}, 2\}$
and $M = \max\{M_1, M_2\}$ where 
$    M_0 = 4 \parenthese{V^2_{\vect{d}} + \sigma_1^2} + 128 V^2_{\vect{d}}$,
$    M_1 = \max \curlybracket{5^{2/3} \parenthese{V_{\vect{d}} + \frac{2}{4^{2/3}} G_0}^2 , M_0}$
and
$    M_2 = 2.55\parenthese{V^2_{\vect{d}} + \sigma^2_1} + \dfrac{28 V^2_{\vect{d}}}{3}$
\end{theorem}


\begin{theorem}
\label[theorem]{thm:submod}
Given a convex set $\mathcal{K}$ and assume that the function $F_t$ is monotone continuous DR-submodular. Setting $Q=T^{2/5}, K=T^{3/5}, T=QK$ and step-size $\eta_k = \frac{1}{K}$. Let $\rho_k = \frac{2}{\parenthese{k+3}^{2/3}}$ and $\rho_k = \frac{1.5}{\parenthese{K-k+2}^{2/3}}$ when $1 \leq k \leq \frac{K}{2}+1$ and $\frac{K}{2} +1 \leq k \leq K$ respectively. Then, the expected $\parenthese{1-\frac{1}{e}}$-regret is at most 
\begin{align}
    \E \bracket{\mathcal{R}_T} \leq \frac{3}{2}\beta D^2 T^{2/5} + \parenthese{C + 3D(N+\sqrt{M}}T^{4/5}    
\end{align}
where the constants are defined in \Cref{thm:convex}
\end{theorem}

As stated in the preceding paragraph, the distinction between convex and submodular optimization can be found in line \ref{update-x} of \Cref{algo:online-dist-FW} and in the oracle optimization subroutine. To achieve the regret bound mentioned in \Cref{thm:convex,thm:submod}, we use follow-the-perturbed-leader as the oracle with regret $\mathcal{R}_T = C\sqrt{T}$. In the case of convex optimization, one may use online gradient descent to obtain the same outcome, but this method is more computationally intensive because it involves a projection step onto the constraint set.
\section{Bandit Setting}
\label{chap:bandit}


This section describes a bandit algorithm for a decentralized submodular maximization. We let $\mathcal{K}$ be a down-closed convex set. A major difference between this algorithm and the previous one is the function's value $f^i_{t}(\vect{x}^i_t)$ being the only information provided to the agent. It does not know of the value incurred if it had chosen another decision in the constraint set. As a consequence, this setting makes access to the gradient impossible for the agent. To circumvent this limitation, we use the one-point gradient estimate \cite{FlaxmanKalai05:Online-convex} and adapt the biphasic bandit setting \cite{Zhang:2019} to our decentralized algorithm.

We recall that for a function $f_t$ defined on $\mathcal{K} \subset \mathbb{R}^d$, it admits a $\delta$-smoothed version for any $\delta > 0$, given as 
\begin{align*}
    \hat{f}_{t,\delta} (\vect{x}_t) = \E_{\vect{v} \sim \mathbb{B}_d} \bracket{f_t(\vect{x}_t + \delta \vect{v})}
\end{align*}
where $\vect{v}$ is drawn uniformly random from the $d$-dimensional unit ball. The value of $\hat{f}_{t,\delta}$ at a point $\vect{x}$ is the average of $f_t$ evaluated across the $d$-dimensional ball of radius $\delta$ centered at $\vect{x}$. This function inherits various functional properties from $f_t$, therefore becomes a suitable approximation for $f_t$, as shown in the following lemma.

\begin{lemma}[Lemma 2 \cite{ChenZHK20}, Lemma 6.6 \cite{Hazanothers16:Introduction-to-online}]
\label[lemma]{lmm:one-point-grad}
Let $f$ be a monotone continuous DR-submodular function. If $f$ is $\beta$-smooth, $G$-Lipschitz, then so is $\hat{f}_\delta$ and we have $\norm{\hat{f}_{\delta}(\vect{x}) - f (\vect{x})} \leq \delta G$. More over, if we choose \vect{u} uniformly from the unit sphere $\mathbb{S}^{d-1}$, the following equation holds 
\begin{align}
    \nabla \hat{f}_{t,\delta} (\vect{x}) = \E_{\vect{u} \sim \mathbb{S}_{d-1}}\bracket{\frac{d}{\delta} f_t (\vect{x} + \delta \vect{u}) \vect{u}} \label{eq:grad-one-point}
\end{align}
\end{lemma}
\Cref{lmm:one-point-grad} shows that a decision that maximizes $\hat{f}_{t,\delta}$ can also maximizes $f_t$ approximately. The $\delta$-smooth version additionally provides a one-point gradient estimate that can be used to estimate the gradient of $f_t$ by evaluating the function at a random point on the $(d-1)$-dimensional sphere of radius $\delta$. It is important to note that the point $\vect{x} + \delta \vect{u}$ may be outside the set when \vect{x} is near to the constraint set's boundary. For this reason, we let $\mathcal{K}' \subset \mathcal{K}$ be the $\delta$-interior of $\mathcal{K}$ that verifies : $\forall \vect{x} \in \mathcal{K}'$, $\mathbb{B}(\vect{x}, \delta) \subset \mathcal{K}$, and solve the optimization problem on the new set $\mathcal{K}'$. By shrinking the constraint set down to $\mathcal{K}'$, we assure that the point $\vect{x}+\delta \vect{u}$ is in $\mathcal{K}$ for any point $\vect{x}$ in $\mathcal{K}$'. More over, if the distance $d(\mathcal{K'}, \mathcal{K})$ between $\mathcal{K}'$ and $\mathcal{K}$ is small enough, we can approximately get the optimal regret bound on the original constraint set $\mathcal{K}$ by running the bandit algorithm on $\mathcal{K}'$. The detail on the construction of $\mathcal{K}'$ is given is \Cref{lmm:discrepancy}

The biphasic setting consist of partitioning $T$ into $Q$ blocks of size $L$, with each block consisting of two phases: exploration and exploitation. 
Each agent $i$ performs $K < L$ steps of exploration by updating the decision vector $\vect{x}^i_{q,k}$ using \cref{eq:submodular-update}.
During the exploration phase, rather than playing the final decision as in \Cref{algo:online-dist-FW}, the agent draws uniformly a random vector $\vect{u}^i_{q,k}$ from $\mathbb{S}_{d-1}$ and plays $\vect{x}^i_{q,k} + \delta \vect{u}^i_{q,k}$ for the function $f^i_{\sigma_q (k)}$, as it can only estimate the gradient at the point it plays. The gradient estimate $\widetilde{\vect{h}}^i_{q,k}$ is then computed using \cref{eq:grad-one-point}, followed by a local aggregation and variance reductions steps, the final step consisting of feeding the variance reduction vector $\widetilde{\vect{a}}^i_{q,k}$ back to the oracle $\mathcal{O}^i_k$. The remaining $L-K$ iterations are used for exploitation, where each agent plays the final decision $x^i_q$ to obtain a high reward. We give the detail of every step in \Cref{algo:online-bandit}. \Cref{chap:bandit_analysis} contains the analysis of \Cref{thm:regret_bandit}

\begin{algorithm}[ht!]
\begin{flushleft}
\textbf{Input}:  Smoothing radius $\delta$, $\delta$-interior $\mathcal{K}'$ with lower bound $\underline{u}$, 
	a time horizon $T$, a block size $L$, number of exploration step $K$. Online linear optimization oracles $\mathcal{O}_{i,1}, \ldots, \mathcal{O}_{i,K}$ for each player $1 \leq i \leq n$, 
	step sizes $\eta_k, \rho_k \in (0, 1)$ for all $1 \leq k \leq K$, number of blocks $Q=T/L$
\end{flushleft}
\begin{algorithmic}[1]
\STATE Initialize linear optimizing oracle $\mathcal{O}^i_k$ for all $1 \leq k \leq K$
\FOR {$q = 1$ to $Q$}	 		
	\FOR{every agent $1 \leq i \leq n$}	%
		\STATE Initialize $\vect{x}^i_{q,1} \gets \underline{u}$ and set $\widetilde{\vect{{a}}}^t_{i,0} \gets 0$ 
		\STATE Update $\vect{x}^i_{q,k}$ using line 5 to 10 of \Cref{algo:online-dist-FW}. Choose $\vect{x}^{i}_{q} \gets \vect{x}^{i}_{q,K+1}$
		%
		\STATE Let $\sigma_q$ be a random permutation of $1, \ldots, L$ --- times in phase $q$.
		\FOR{$1 \leq \ell \leq L$}
		    \STATE Let $s = \sigma_q^{-1}(\ell)$ 
		    \IF{$\ell \leq K$}
		        \STATE play $f^i_{q,\ell} \parenthese{\vect{x}^i_{q,s} + \delta \vect{u}^i_{q,s}}$ where $\vect{u}^i_{q,s} \in \mathbb{S}^{d-1}$. - \textit{Exploration}
		    \ELSE
		        \STATE play $f^i_{q,\ell} \parenthese{\vect{x}^i_{q}}$. - \textit{Exploitation}
		    \ENDIF
		\ENDFOR
		\STATE Set $\widetilde{\vect{g}}^{i}_{q,1} \gets \frac{d}{\delta} f^i_{\sigma_q(1)} \parenthese{\vect{x}^i_{q,1} + \delta \vect{u}^i_{q,1}}\vect{u}^i_{q,1}$
			\FOR{$1 \leq k \leq K$}
			    \STATE Let $\widetilde{\vect{h}}^i_{q,k} = \frac{d}{\delta} f^i_{\sigma_q (k)} \parenthese{\vect{x}^i_{q,k} + \delta \vect{u}^i_{q,k}} \vect{u}^i_{q,k}$
				\STATE Send $\widetilde{\vect{g}}^{i}_{q,k}$ to all neighbours $N(i)$.
				\STATE After receiving $\widetilde{\vect{g}}^{j}_{q,k}$ from all neighbours $j \in N(i)$, compute
					$\widetilde{\vect{d}}^{i}_{q,k} \gets  \sum_{j \in N(i)} W_{ij} \widetilde{\vect{g}}^{j}_{q,k}$
				\STATE $\widetilde{\vect{g}}^{i}_{q,k + 1} \gets \widetilde{\vect{h}}^i_{q,k+1} - \widetilde{\vect{h}}^i_{q,k} + \widetilde{\vect{d}}^{i}_{q,k}$
				\STATE $\widetilde{\vect{a}}^i_{q,k} \gets (1 - \rho_k) \cdot \widetilde{\vect{a}}^i_{q, k-1} + \rho_k \cdot \widetilde{\vect{d}}^{i}_{q, k}$.
				\STATE Feedback function $\langle \widetilde{\vect{a}}^{i}_{q,k}, \cdot \rangle$ 
				to oracles $\mathcal{O}^i_k$. (The cost of the oracle $\mathcal{O}^i_k$ at block $q$ is 
				$\langle \widetilde{\vect{a}}^{i}_{q,k}, \vect{v}^{i}_{q,k}  \rangle$.)
			\ENDFOR
	\ENDFOR
\ENDFOR
\end{algorithmic}
\caption{Bandit Monode Frank-Wolfe}
\label{algo:online-bandit}
\end{algorithm}

\begin{lemma}[Lemma 1, \cite{Zhang:2019}]
\label[lemma]{lmm:discrepancy}
Let $\mathcal{K}$ is down-closed convex set and $\delta$ is is sufficiently small such that $\alpha = \frac{\sqrt{d}+1}{r}\delta < 1$. The set $\mathcal{K}' = (1-\alpha)\mathcal{K} + \delta \mathbf{1}$ is convex, compact and down-closed $\delta$-interior of $\mathcal{K}$ satisfies $d(\mathcal{K}, \mathcal{K}') \leq \parenthese{\sqrt{d}\parenthese{\frac{R}{e} +1} + \frac{R}{r}} \delta$ 
\end{lemma}


\begin{theorem}
\label{thm:regret_bandit}
Let $\mathcal{K}$ be a down-closed convex and compact set. We suppose the $\delta$-interior $\mathcal{K}'$ verify $\Cref{lmm:discrepancy}$. Let $Q = T^{2/9}, L = T^{7/9}, K = T^{2/3} $, $\delta = \frac{r}{\sqrt{d}+2}T^{-1/9}$ and $\rho_k = \frac{2}{(k+2)^{2/3}}$, $\eta_k = \frac{1}{K}$. Then the expected $\parenthese{1-\frac{1}{e}}$-regret is upper bounded 
\begin{align}
    \E \bracket{\mathcal{R}_T} \leq ZT^{8/9} + \frac{\beta D^2}{2}T^{1/9} 
                + \frac{3}{2} D \frac{d \parenthese{\sqrt{d}+2}}{r} P_{n,\lambda_2} T^{2/9} + \beta D^2 T^{3/9}
\end{align}
where we note
$    Z = \parenthese{1-\frac{1}{e}} \parenthese{\sqrt{d}\parenthese{\frac{R}{e} +1} + \frac{R}{r}} G \frac{r}{\sqrt{d}+2} 
            + \parenthese{2-\frac{1}{e}}G \frac{r}{\sqrt{d}+2}+ 2\beta + C$
and 
$
    P_{n,\lambda_2} = k_0 \cdot n B\max \curlybracket{\lambda_2\parenthese{1 + \frac{2}{1-\lambda_2}}, 2} 
    + 4^{1/3}\parenthese{24n^2 \parenthese{\frac{1}{\frac{1}{\lambda_2}-1} + 1}^2 + 8n\parenthese{\frac{1}{\parenthese{\frac{1}{\lambda_2}-1}^2}+2}}^{1/2}
$
\end{theorem}
\section{Experiments}
\label{chap:experiment}


We run the algorithm on a movie recommendation problem, with the goal of identifying a set of $k$ movies that satisfy all users. Our setting is closely related to the one in \cite{MokhtariHK18} and \cite{xie19b}. We use the MovieLens dataset, which contains one million ratings ranging from 1 to 5 from 6000 users on 3883 movies. We divided the data set into $T$ batches $B_1, \dots, B_T$, with each batch $B_t$ containing ratings from 50 users. We chose Complete, Line, Grid, and Erdos-Renyi graphs with linked probability $0.2$. We set the number of nodes/agents equals to $10$, $25$, and $50$. At each iteration $t$, the agent $i$ receives a subset of ratings $B^i_t \subset B_t$. Let $\mathcal{M}$ be the set of movies and $\mathcal{U}$ the set of users; we note $r(u,m)$ the rating of user $u \in \mathcal{U}$ for movies $m \in \mathcal{M}$. Let $S \subset \mathcal{M}$ a collection of movies such that $|S| = k$, the facility location function associated to each agent $i$ denoted, 
\begin{align}
\label{pb:facility_location}
    f(B^i_t, S) = f^i_t(S) = \frac{1}{|B^i_t|} \sum_{u \in B^i_t} \max_{m \in S} r(u,m)
\end{align}
We denote by $\mathcal{K} = \curlybracket{\vect{x} \in \bracket{0,1}^d \vert \sum_{j=1}^d \vect{x}_j = k}$. The multilinear extension of $f^t_i$ is defined as,
\begin{align}
    F^i_t (\vect{x}) = \sum_{S \subset \mathcal{M}} f^i_t (S) \prod_{j \in S} \vect{x}_j \prod_{\ell \not\in S} \parenthese{1-\vect{x}_\ell}, \quad \forall \vect{x} \in \mathcal{K}
\end{align}
The goal is to maximize the global objective function $F_t(\vect{x}) = \frac{1}{n} \sum_{i=1}^n F^i_t (\vect{x})$, subject to $\vect{x} \in \mathcal{K}$ while using only local communication and partial information for each local functions.

\begin{figure}[htb!]
     \centering
     \subfloat[$(1-1/e)$-Regret on Complete Graph]{\label{fig:regret-k}
         \includegraphics[width=.45\textwidth]{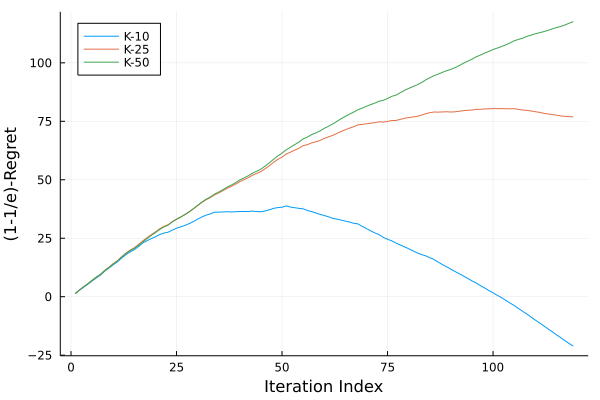}}
     \subfloat[Ratio on Complete graph]{\label{fig:ratio-k}
         \includegraphics[width=.45\textwidth]{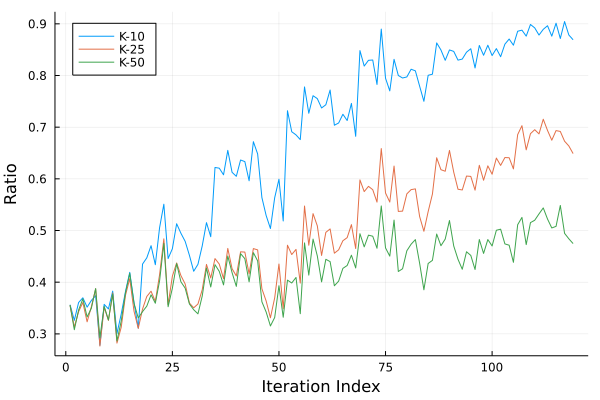}}
        \caption{Algorithm performance on complete graph. The number of nodes is 10, 25 and 50. }
        \label{fig:regret}
\end{figure}



\Cref{fig:regret-k} shows the $\parenthese{1-\frac{1}{e}}$-regret of the algorithm for $k=20$ on a complete graph with different node's configuration. We observe that increasing network size leads to a decrease in regret value, which is expected in a decentralized setting because information distributed across a larger set of nodes makes reaching consensus more difficult. Recall that the algorithm uses the same value for each function $f_t$ in block $q$. If we set $K = 17$ and $Q = 6$, we can expect a stepwise-like curve since the objective function's value changes significantly at each round $t \mod 17$. In a small graph configuration, this value change is more pronounced, bringing the cumulative sum of the objective function closer to the $\parenthese{1-\frac{1}{e}}$-optimal value.
\Cref{fig:ratio-k} depicts the ratio of our algorithm's objective value on a complete graph to an offline centralized Frank-Wolfe. As $t$ increases, the ratio approaches one, demonstrating that our algorithm's performance is comparable to that of an offline setting if we run the algorithm for many rounds, particularly in a $10$-nodes configuration. Thus, the results validate our theoretical analysis in the previous section. 

\begin{figure}[ht!]
     \centering
         \includegraphics[scale=0.4]{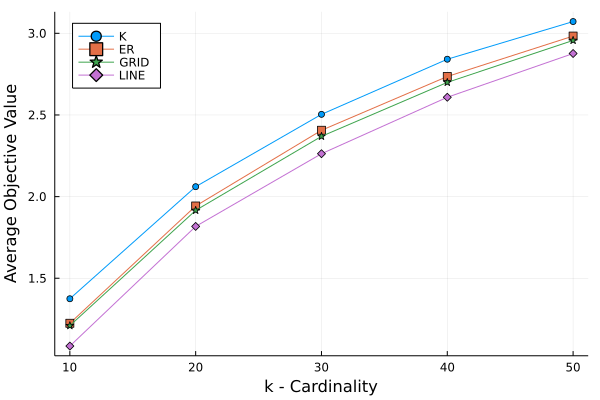}
         \caption{Average objective value over $T$ rounds in function of the cardinality constraint.}
         \label{fig:vary-k}
\end{figure}

\Cref{fig:vary-k} shows the average value of the objective function over $T$ rounds for all graph types when the number of movie suggestions $k$ is varied in a $50$-node configuration. The average degree for Erdos-Renyi, Complete, Grid, and Line is $5.8, 51, 5.4$, and $4$, respectively. As a result, we observe lower performance on less connected graphs when compared to other graph settings. We also notice that increasing the value of $k$ is equivalent to relaxing the cardinality constraint, which results in better performance on the objective function.

\section{Concluding remarks}
\label{chap:conclusion}


In this paper, we propose a decentralized online algorithm for optimizing convex and monotone continuous DR-submodular functions with regret and $\parenthese{1-\frac{1}{e}}$-regret bound of $O(T^{4/5})$. The extension of the algorithm to the bandit setting ensures a $\parenthese{1-\frac{1}{e}}$-regret bound of $O(T^{8/9})$. A detailed analysis is given when the constraint set is either a general convex set or a downward-closed convex set under full information and bandit settings, respectively. In addition, the experiment results on a real-life movie recommendation problem assess the interest of the proposed algorithm for learning in decentralized settings.


\bibliographystyle{unsrt}  
\bibliography{ref} 

\newpage
\appendix
\section*{Supplementary File for One Gradient Frank-Wolfe for Decentralized Online Convex and Submodular Optimization} 
\section{Theoretical Analysis for \Cref{chap:formulation}}
\label{chap:analysis}

In the analysis, we note $\sigma_q (k)$ to be the permutation of $k$ at phase $q$. We define the average function of the remaining \emph{$(K-k)$} functions as
\begin{align}
    \bar{F}_{q,k} (\vect{x}) = \frac{1}{K-k} \sum_{\ell=k+1}^K F_{\sigma_q (\ell)} (\vect{x}) = \frac{1}{K-k} \sum_{\ell=k+1}^K \frac{1}{n} \sum_{i=1}^n f^i_{\sigma_q (\ell)} (\vect{x}) \nonumber \\
\end{align}
where $F_{\sigma_q (\ell)} (\vect{x}) = \frac{1}{n} \sum_{i=1}^n f^i_{\sigma_q (\ell)} (\vect{x})$. 
We also define 
\begin{align}
    \hat{f}^i_{q,k} = \frac{1}{K-k} \sum_{\ell=k+1}^K f^i_{\sigma_q (\ell)} (\vect{x}^i_{q,\ell}),
    \quad \nabla \hat{f}^i_{q,k} = \frac{1}{K-k} \sum_{\ell=k+1}^K \nabla f^i_{\sigma_q (\ell)} (\vect{x}^i_{q,\ell})
\end{align}
as the average of the remaining \emph{$(K-k)$} functions and stochastic gradients of $f^i_{\sigma_q (\ell)} (\vect{x}^i_{q,\ell}$) respectively. Then we note, 
\begin{align}
    \hat{F}_{q,k} = \frac{1}{n}\sum_{i=1}^n \hat{f}^i_{q,k}, 
    \quad \nabla \hat{F}_{q,k} = \frac{1}{n}\sum_{i=1}^n \nabla \hat{f}^i_{q,k}, 
\end{align}
In the same spririt of $\hat{f}^i_{q,k}$, we define 
\begin{align}
    \hat{\vect{g}}^i_{q,k} = \frac{1}{K-k} \sum_{\ell=k+1}^K \vect{g}^i_{q,\ell},
    \quad \hat{\vect{d}}^i_{q,k} = \frac{1}{K-k} \sum_{\ell=k+1}^K \vect{d}^i_{q,\ell}
\end{align}
In the rest of the analysis, we let $\mathcal{F}_{q,1} \subset \dots \subset \mathcal{F}_{q,k}$ to be the $\sigma$-field generated by the permutation up to time $k$ and $\mathcal{H}_{q,1} \subset \dots \subset \mathcal{H}_{q,k}$ another $\sigma$-field generated by the randomness of the stochastic gradient estimate up to time $k$.

\begin{assumption}
\label{assum:assum_d_var_red}
Let $\curlybracket{\widetilde{\vect{d}}_{t}}_{1}^T$ be a sequence such that $\E\bracket{\widetilde{\vect{d}}_{t}\bigm\vert \mathcal{H}_{t-1}} = \vect{d}_{t}$ where $\mathcal{H}_{t-1}$ is the filtration of the stochastic estimate up to $t-1$.
\end{assumption}

\begin{lemma}[Fact 1, \cite{WaiLafond17:Decentralized-Frank--Wolfe}]
\label[lemma]{lmm:average_consensus}
Let $\vect{x}^1, \ldots, \vect{x}^n \in \mathbb{R}^d$ be a set of vector and and $\vect{x}_{avg} := \frac{1}{n} \sum_{i=1}^n \vect{x}_i$ their average. Let $\mathbf{W}$ be non-negative doubly stochastic matrix. The output of one round of the average consensus update $\Bar{\vect{x}}^i = \sum_{j=1}^n W_{ij} \vect{x}^j$ satisfy :
\begin{align*}
    \sqrt{\sum_{i=1}^n \norm{\Bar{\vect{x}}^i - \vect{x}_{avg}}^2} 
    \leq |\lambda_2 (\mathbf{W})| \cdot \sqrt{\sum_{i=1}^n \norm{\vect{x}^i - \vect{x}_{avg}}^2}
\end{align*}
where $\lambda_2 (\mathbf{W})$ is the second largest eigenvalue of $\mathbf{W}$.
\end{lemma}

\begin{lemma}
\label[lemma]{lmm:x_iteration}
For all $1 \leq q \leq Q$ and $1 \leq k \leq K$, we have
\begin{align}
    \overline{\vect{x}}_{q,k+1} = \overline{\vect{x}}_{q,k} + \eta_k \parenthese{\frac{1}{n}\sum_{i=1}^n \vect{v}^i_{q,k} - \overline{\vect{x}}_{q,k}}
\end{align}
for convex case, and 
\begin{align}
    \overline{\vect{x}}_{q,k+1} = \overline{\vect{x}}_{q,k} + \frac{1}{K} \parenthese{\frac{1}{n}\sum_{i=1}^n \vect{v}^i_{q,k}}
\end{align}
for submodular case.
\end{lemma}

\begin{proof}
\begin{align*}
\overline{\vect{x}}_{q,k+1} 
&= \frac{1}{n} \sum_{i=1}^{n} \vect{x}^{i}_{q,k + 1} \tag{Definition of $\overline{\vect{x}}_{q,k+1}$}
\\
&= \frac{1}{n} \sum_{i=1}^{n} \left ((1-\eta_{k}) \vect{y}^{i}_{q,k} + \eta_{k} \vect{v}_{q,k}^i \right )  \tag{Definition of $\vect{x}_{q,k}^i$} 
\\
&= \frac{1}{n} \sum_{i=1}^{n} \left [ (1-\eta_{k}) \left ( \sum_{j=1}^n \vect{W}_{ij} \vect{x}_{q,k}^j  \right ) + \eta_{k} \vect{v}_{q,k}^i \right ] \tag{Definition of $\vect{y}_{q,k}^i$}
\\
&= (1 - \eta_{k}) \frac{1}{n} \sum_{i=1}^n \left [ \sum_{j=1}^n \vect{W}_{ij}\vect{x}_{q,k}^j \right ] + \frac{1}{n}\eta_{k} \sum_{i=1}^n \vect{v}_{q,k}^i
\\
&= (1 - \eta_{k}) \frac{1}{n} \sum_{j=1}^n \left [ \vect{x}_{q,k}^j \sum_{i=1}^n \vect{W}_{ij} \right ] + \frac{1}{n} \eta_{k} \sum_{i=1}^n \vect{v}_{q,k}^i
\\
&= (1 - \eta_{k}) \frac{1}{n} \sum_{j=1}^n \vect{x}_{q,k}^j + \frac{1}{n} \eta_{k} \sum_{i=1}^n \vect{v}_{q,k}^i \tag{$\sum_{i=1}^n W_{ij} =1$ for every $j$}
\\
&= (1 - \eta_{k}) \overline{\vect{x}}_{q,k} + \frac{1}{n} \eta_{k} \sum_{i=1}^n \vect{v}_{q,k}^i
\\
&= \overline{\vect{x}}_{q,k} + \eta_{k} \left ( \frac{1}{n} \sum_{i=1}^n \vect{v}_{q,k}^i - \overline{\vect{x}}_{q,k} \right )
\end{align*}
where we use a property of $W$ which is $\sum_{i = 1}^{n} W_{ij} = 1$ for every $j$. The proof of the second equation is similar.
\end{proof}

\begin{lemma}
\label[lemma]{lmm:condition_n_equal_hat}
For all $k \in \curlybracket{1, \ldots, K}$, it holds that
\begin{align}
    \nabla \hat{F}_{q,k} 
    = \frac{1}{n} \sum_{i=1}^n  \nabla \hat{f}^i_{q,k} 
    = \frac{1}{n} \sum_{i=1}^n   \hat{\vect{g}}^i_{q,k} 
    = \frac{1}{n} \sum_{i=1}^n  \hat{\vect{d}}^i_{q,k}
\end{align}
\end{lemma}

\begin{proof}
First, we verify that $\forall \ell \in \curlybracket{1, \ldots, K}$
\begin{align}
    \label{eq:condition_n_equal}
    \frac{1}{n} \sum_{i=1}^n \nabla f^i_{\sigma_q (\ell)} (\vect{x}^i_{q,\ell}) 
    = \frac{1}{n} \sum_{i=1}^n \vect{g}^i_{q,\ell}
    =  \frac{1}{n} \sum_{i=1}^n \vect{d}^i_{q,\ell}
\end{align} 
For the base step $\ell=1$, we have $\vect{g}^i_{q,\ell} = \nabla f^i_{\sigma_q (1)} (\vect{x}^i_{q,1})$. Averaging over $n$ yields
\begin{align}
    \frac{1}{n} \sum_{i=1}^n \vect{g}^i_{q,1} 
    = \frac{1}{n} \sum_{i=1}^n \nabla f^i_{\sigma_q (1)} (\vect{x}^i_{q,1}) \nonumber
\end{align}
Since $\mathbf{W}$ is doubly stochastic, we have
\begin{align}
    \label{eq:doubly_stoc}
    \frac{1}{n} \sum_{i=1}^n \vect{d}^i_{q,1}
    = \frac{1}{n} \sum_{i=1}^n \sum_{j=1}^n W_{ij} \vect{g}^j_{q,1}
    = \frac{1}{n} \sum_{j=1}^n\vect{g}^j_{q,1}  \sum_{i=1}^n W_{ij}
    = \frac{1}{n} \sum_{j=1}^n\vect{g}^j_{q,1}
\end{align}
For the recurrence step, recall the definition of $\vect{g}^i_{q,\ell}$,
\begin{align*}
    \vect{g}^i_{q, \ell}
    =   \nabla f^i_{\sigma_q(\ell)} (\vect{x}^i_{q,\ell}) 
        - \nabla f^i_{\sigma_q(\ell-1)} (\vect{x}^i_{q,\ell-1}) 
        + \vect{d}^i_{q,\ell-1}
\end{align*}
Averaging over $n$ and using the recurrence hypothesis $\frac{1}{n}\sum_{i=1}^n\nabla f^i_{\sigma_q (\ell-1)} (\vect{x}^i_{q,\ell-1}) =  \frac{1}{n}\sum_{i=1}^n\vect{d}^i_{q,\ell-1}$, we deduce that 
\begin{align}
    \frac{1}{n} \sum_{i=1}^n \vect{g}^i_{q,\ell}
    =  \frac{1}{n} \sum_{i=1}^n \nabla f^i_{\sigma_q (\ell)} (\vect{x}^i_{q,\ell}) 
\end{align}
Also, using the same techniques in \cref{eq:doubly_stoc} for $\vect{d}^i_{q,\ell}$, we complete the verification for \cref{eq:condition_n_equal}. The proof of \Cref{lmm:condition_n_equal_hat} can be deduced from \cref{eq:condition_n_equal} by averaging over $\ell \in \curlybracket{k+1, \ldots, K}$
\end{proof}

\begin{lemma}
\label[lemma]{lmm:bound_d_avg_consensus}
Suppose that each of $f^i_{\sigma_q (k)}$ is \emph{$\beta$}-smooth. Using the Frank-Wolfe update of $\vect{x}^i_{q,k}$, the average of the remaining $(K-k)$ gradient approximation $\hat{\vect{d}}^i_{q,k}$ satisfies 
    \begin{align*}
         \max_{i \in \bracket{1,n}}\E \bracket{ \norm{
            \hat{\vect{d}}^i_{q,k} - \nabla \hat{F}_{q,k}
        }}
    \leq \begin{dcases}
        \dfrac{N}{k} \qquad k \in \bracket{1,\dfrac{K}{2}} \\
        \dfrac{N}{K-k+1} \qquad k \in \bracket{\dfrac{K}{2}+1,K}
    \end{dcases}
    \end{align*}
where $N = nGk_0 \max \{\lambda_2 \parenthese{1 + \frac{2}{1-\lambda_2}}, 2\}$.
\end{lemma}

\begin{proof}
We will prove the lemma by induction following the idea from Lemma 2 of \cite{WaiLafond17:Decentralized-Frank--Wolfe}. Let's define following variables
\begin{align}
\label{def:cat_vec}
    \vect{\hat{d}}^{cat}_{q,k} = \bracket{\vect{\hat{d}}^{1 \top}_{q,k}, \dots, \vect{\hat{d}}^{n \top}_{q,k} }^{\top},
    \quad \vect{\hat{g}}^{cat}_{q,k} = \bracket{\vect{\hat{g}}^{1 \top}_{q,k}, \dots, \vect{\hat{g}}^{n \top}_{q,k} }^\top,
    \quad \nabla \hat{F}^{cat}_{q,k} = \bracket{\nabla \hat{F}^{\top}_{q,k}, \dots, \nabla \hat{F}^{\top}_{q,k} }^{\top}
\end{align}
and let the slack variables as 
\begin{align}
    \delta^i_{q,k} := \nabla \hat{f}^i_{q,k} - \nabla \hat{f}^i_{q,k-1}, 
    \quad \bar{\delta}_{q,k} := \frac{1}{n} \sum_{i=1}^n \parenthese{\nabla \hat{f}^i_{q,k} - \nabla \hat{f}^i_{q,k-1}} = \nabla\hat{F}_{q,k} - \nabla\hat{F}_{q,k-1} 
\end{align}
then, following the definition in \ref{def:cat_vec}, we note
\begin{align*}
    \delta^{cat}_{q,k} = \bracket{\delta^{1 \top}_{q,k}, \dots, \delta^{n \top}_{q,k}}^{\top},
    \quad \bar{\delta}^{cat}_{q,k} = \bracket{\bar{\delta}^{\top}_{q,k}, \dots, \bar{\delta}^{\top}_{q,k}}^{\top}
\end{align*}
By \Cref{lmm:average_consensus}, we have
\begin{align}
\label{eq:gradient_consensus}
    \norm{\vect{\hat{d}}^{cat}_{q,k} - \nabla \hat{F}^{cat}_{q,k}}^2
    &= \sum_{i=1}^n \norm{\vect{\hat{d}}^{i}_{q,k} - \nabla \hat{F}_{q,k}}^2 \nonumber \\
    & \leq \lambda_2^2 \sum_{i=1}^n \norm{\vect{\hat{g}}^{i}_{q,k} - \nabla \hat{F}_{q,k}}^2 \nonumber \\
    & = \lambda_2^2 \norm{\vect{\hat{g}}^{cat}_{q,k} - \nabla \hat{F}^{cat}_{q,k}}^2
\end{align}
We can deduce that 
\begin{align}
\label{eq:recurrence_d_F}
    \E \bracket{\norm{\vect{\hat{d}}^{cat}_{q,k} - \nabla \hat{F}^{cat}_{q,k}}}
    &\leq \lambda_2 \E \bracket{\norm{
            \vect{\hat{g}}^{cat}_{q,k} - \nabla \hat{F}^{cat}_{q,k}
            } }\nonumber \\
    &=  \lambda_2 \E \bracket{\norm{
            \delta^{cat}_{q,k} + \vect{\hat{d}}^{cat}_{q,k-1} 
            - \nabla \hat{F}^{cat}_{q,k} + \nabla \hat{F}^{cat}_{q,k-1} - \nabla \hat{F}^{cat}_{q,k-1}
            }} \nonumber \\
    & \leq \lambda_2 \parenthese{
        \E \bracket{\norm{\vect{\hat{d}}^{cat}_{q,k-1} - \nabla \hat{F}^{cat}_{q,k-1}}}
        + \E \bracket{\norm{\delta^{cat}_{q,k} - \bar{\delta}^{cat}_{q,k}}} 
    } \nonumber \\
    & \leq \lambda_2 \parenthese{
         \E \bracket{\norm{\vect{\hat{d}}^{cat}_{q,k-1} - \nabla \hat{F}^{cat}_{q,k-1}}}
        + \E \bracket{\norm{\delta^{cat}_{q,k}}}  
        }
\end{align}
since 
\begin{align}
    \norm{\delta^{cat}_{q,k} - \bar{\delta}^{cat}_{q,k}}^2  
    = \sum_{i=1}^n \norm{\delta^{i}_{q,k} - \bar{\delta}_{q,k}}^2  
    \leq \sum_{i=1}^n \norm{\delta^{i}_{q,k}}^2 - n\norm{\bar{\delta}_{q,k}}^2
    \leq \sum_{i=1}^n \norm{\delta^{i}_{q,k}}^2
    = \norm{\delta^{cat}_{q,k}}^2
\end{align}
Notice that we can bound the expected value of $\delta^{cat}$ by 
\begin{align}
\label{eq:diff_k}
    \E \bracket{\norm{\delta^{cat}_{q,k}}^2}
    &= \E \bracket{\sum_{i=1}^n \norm{\delta ^i_{q,k}}^2} 
    = \sum_{i=1}^n \E  
    \bracket{
        \norm{
            \nabla \hat{f}^i_{q,k} - \nabla \hat{f}^i_{q,k-1}
        }^2
    } \nonumber\\
    & \quad = \sum_{i=1}^n \E 
        \bracket{
            \E 
                \bracket{ 
                    \norm{
                        \nabla \hat{f}^i_{q,k} - \nabla \hat{f}^i_{q,k-1}
                    }^2
                    \bigm\vert \mathcal{F}_{q,k-1}
                }
        } \nonumber\\
    & \quad = \sum_{i=1}^n \E \bracket{
            \E 
                \bracket{ 
                    \norm{
                        \frac{\sum_{\ell=k+1}^K \nabla f^i_{\sigma_q (\ell)} (\vect{x}^i_{q,\ell})}{K-k} 
                        - \frac{\sum_{\ell=k}^K \nabla f^i_{\sigma_q (\ell)} (\vect{x}^i_{q,\ell})}{K-k+1}
                        }^2
                        \bigm\vert  \mathcal{F}_{q,k-1}
                }
            } \nonumber \\
    & \quad = \sum_{i=1}^n\E \bracket{
                \E \bracket{ \norm{
                    \frac{\sum_{\ell=k+1}^K \nabla f^i_{\sigma_q (\ell)} (\vect{x}^i_{q,\ell})}{(K-k)(K-k+1)}
                     - \frac{\nabla f^i_{\sigma_q (k)} (\vect{x}^i_{q,k})}{{K-k+1}}}^2
                     \bigm\vert \mathcal{F}_{q,k-1} }
            }
         \nonumber\\
    &  \quad \leq n
            \parenthese{
                \frac{2G}{K-k+1}
            }^2
\end{align}
using Jensen's inequality, we can deduce that
\begin{align}
\label{eq:bound_delta_cat}
    \E \bracket{\norm{\delta^{cat}_{q,k}}}
    \leq \sqrt{\E \bracket{\norm{\delta^{cat}_{q,k}}^2}} 
    \leq \frac{2\sqrt{n}G}{K-k+1}
\end{align}
We are now proving the lemma by induction, when $k=1$, we have
\begin{align}
\label{eq:bound_cat_k1}
    \E& \bracket{\norm{\vect{\hat{d}}^{cat}_{q,1} - \nabla \hat{F}^{cat}_{q,1}}^2}
    = \E \bracket{\sum_{i=1}^n \norm{
        \vect{\hat{d}}^{i}_{q,1} - \nabla \hat{F}_{q,1}}^2
    } \nonumber 
    \leq \lambda_2^2 \E \bracket{\sum_{i=1}^n \norm{
        \vect{\hat{g}}^{i}_{q,1} - \nabla \hat{F}_{q,1}}^2
    } \nonumber \\
    & \leq \lambda_2^2 \E \bracket{\sum_{i=1}^n \norm{
        \nabla \hat{f}^{i}_{q,1} - \nabla \hat{F}_{q,1}}^2
    } \nonumber 
    \leq \lambda_2^2 \E \bracket{
        \sum_{i=1}^n \norm{\nabla \hat{f}^{i}_{q,1}}^2
        }
    \leq n\lambda_2^2 G^2
\end{align}
where we have used Lipschitzness of $f$ in the last inequality. We now suppose that $1 \leq k \leq k_0$, from \cref{eq:recurrence_d_F,eq:bound_delta_cat} 
\begin{align}
    \E \bracket{\norm{\vect{\hat{d}}^{cat}_{q,k} - \nabla \hat{F}^{cat}_{q,k}}}
     & \leq \lambda_2 \parenthese{
         \E \bracket{\norm{\vect{\hat{d}}^{cat}_{q,k-1} - \nabla \hat{F}^{cat}_{q,k-1}}}
        + \E \bracket{\norm{\delta^{cat}_{q,k}}}  
        } \nonumber \\
    & \leq \lambda_2^{k-1}\sqrt{n}G + 2 \sum_{\tau=1}^{k} \lambda_2^{\tau} \sqrt{n}G \nonumber\\
    & \leq \lambda_2\sqrt{n}G + 2\frac{\lambda_2}{1-\lambda_2}\sqrt{n}G \nonumber \\
    &= \lambda_2 \sqrt{n}G \parenthese{1 + \frac{2}{1-\lambda_2}}
\end{align}
We set $N_0 = k_0\sqrt{n}G\max \{\lambda_2 \parenthese{1 + \frac{2}{1-\lambda_2}}, 2\}$, we claim that  $\E \bracket{\norm{\vect{\hat{d}}^{cat}_{q,k} - \nabla \hat{F}^{cat}_{q,k}}} \leq \dfrac{N_0}{k}$ for $k \in \bracket{k_0,\frac{K}{2}+1}$. Recall that $K-k+1 \geq k-1$, by \cref{eq:recurrence_d_F,eq:bound_delta_cat} and induction hypothesis, we have
\begin{align}
    \E \bracket{\norm{\vect{\hat{d}}^{cat}_{q,k} - \nabla \hat{F}^{cat}_{q,k}}}
    & \leq \lambda_2 \parenthese{
         \E \bracket{\norm{\vect{\hat{d}}^{cat}_{q,k-1} - \nabla \hat{F}^{cat}_{q,k-1}}}
        + \E \bracket{\norm{\delta^{cat}_{q,k}}}  
        } \nonumber \\
    & \leq \lambda_2 \left( 
        \frac{N_0}{k-1} + \frac{2\sqrt{n}G}{K-k+1}
    \right) \nonumber \\
    & \leq \lambda_2 \left( 
        \frac{N_0}{k-1} + \frac{2\sqrt{n}G}{k-1}
    \right) \nonumber \\
    & \leq \lambda_2 \left( 
        \frac{N_0 + 2\sqrt{n}G}{k-1}
    \right)\nonumber \\
    & \leq \lambda_2 \left( 
        N_0\frac{k_0 + 1}{k_0 (k-1)}
    \right) \nonumber \\
    & \leq \frac{N_0}{k} \label{eq:boundksmall}
\end{align}
where we have used the fact that $\lambda_2(\mathbf{W}) \dfrac{k_0+1}{k_0(k-1)} \leq \dfrac{1}{k}$ in the last inequality.
When $k \in \bracket{\frac{K}{2}+1, K}$, we claim that $\E \bracket{\norm{\vect{\hat{d}}^{cat}_{q,k} - \nabla \hat{F}^{cat}_{q,k}}} \leq \dfrac{N_0}{K-k+1}$. The base case $k = \frac{K}{2}+1$ is verified by \cref{eq:boundksmall},
\begin{align}
    \E \bracket{\norm{\vect{\hat{d}}^{cat}_{q,k} - \nabla \hat{F}^{cat}_{q,k}}}
    \leq \frac{N_0}{\frac{K}{2}+1}
    \leq \frac{N_0}{\frac{K}{2}}
    \leq \frac{N_0}{K-(\frac{K}{2}+1) + 1}
\end{align}
For $k \geq \frac{K}{2}+2$, using \cref{eq:recurrence_d_F,eq:bound_delta_cat} and the induction hypothesis, we have 
\begin{align}
    \label{eq:bound_delta_big}
    \E \bracket{\norm{\vect{\hat{d}}^{cat}_{q,k} - \nabla \hat{F}^{cat}_{q,k}}}
    & \leq \lambda_2 \parenthese{
         \E \bracket{\norm{\vect{\hat{d}}^{cat}_{q,k-1} - \nabla \hat{F}^{cat}_{q,k-1}}}
        + \E \bracket{\norm{\delta^{cat}_{q,k}}}  
        } \nonumber \\
    & \leq \lambda_2 \parenthese{
        \frac{N_0}{K-k+2}
        + \frac{2\sqrt{n}G}{K-k+1}
    } \nonumber \\
    & \leq  \lambda_2\left(
        \frac{N_0 + 2G}{K-k+1}
    \right) \nonumber\\
    & \leq \lambda_2
        \parenthese{N_0 \frac{k_0 + 1}{k_0 (K-k+1)}}
     \nonumber \\ 
    & \leq \frac{N_0}{K-k+1} 
\end{align}
Recall that
\begin{align}
\label{eq:bound_sum_sqrt}
    \frac{1}{\sqrt{n}} \E \bracket{\sum_{i=1}^n \norm{\vect{\hat{d}}^{i}_{q,k} - \nabla \hat{F}_{q,k}}} 
    \leq \E \bracket{\parenthese{\sum_{i=1}^n \norm{\vect{\hat{d}}^{i}_{q,k} - \nabla \hat{F}_{q,k}}^2}^{1/2}}
    = \E \bracket{\norm{\vect{\hat{d}}^{cat}_{q,k} - \nabla \hat{F}^{cat}_{q,k}}}
\end{align}
The desired result followed from \cref{eq:boundksmall,eq:bound_delta_big,eq:bound_sum_sqrt} where $N = \sqrt{n}N_0$
\begin{align}
    \max_{i \in \bracket{1,n}} \E \bracket{ \norm{
        \hat{\vect{d}}^i_{q,k} - \nabla \hat{F}_{q,k}}
    } 
    \leq \begin{dcases}
        \dfrac{N}{k} \qquad k \in \bracket{1, \frac{K}{2}} \\
        \dfrac{N}{K-k+1} \qquad k \in \bracket{\frac{K}{2}+1, K}
    \end{dcases}
\end{align}

\end{proof}

\setcounter{theorem}{0}
\begin{lemma}
For $i \in \bracket{n}, k \in \bracket{K}$. Let $V_d = 2nG \parenthese{\frac{\lambda_2}{1-\lambda_2}+1}$, the local gradient is upper-bounded, i.e $\norm{\vect{d}^i_{q,k}} \leq V_d$
\end{lemma}
\begin{proof}
We use the same notation introduced in \cref{def:cat_vec}. Let's define
\begin{align}
     \vect{d}^{cat}_{q,k} = \bracket{\vect{d}^{1 \top}_{q,k}, \dots, \vect{d}^{n \top}_{q,k}}^{\top} \in \mathbb{R}^{nd},
    \quad \nabla f^{cat}_{\sigma_q(k)} = \bracket{\nabla f^{1 }_{\sigma_q (k)}(\vect{x}^1_{q,k})^\top, \dots, \nabla f^{n }_{\sigma_q (k)}(\vect{x}^n_{q,k})^\top}^{\top} \in \mathbb{R}^{nd}
\end{align}
and 
\begin{align}
    \nabla F^{cat}_{\sigma_q(k)} = \bracket{\nabla F_{\sigma_q(k)}^\top, \dots, \nabla F_{\sigma_q(k)}^\top}^\top
    = \bracket{\frac{1}{n} \sum_{i=1}^n \nabla f^i_{\sigma_q(k)} (\vect{x}^i_{q,k})^\top, \dots, \frac{1}{n} \sum_{i=1}^n \nabla f^i_{\sigma_q(k)} (\vect{x}^i_{q,k})^\top}^\top 
\end{align}
Using the local gradient update, we have
\begin{align}
\label{eq:bound_dk_step1}
    \vect{d}^{cat}_{q,k} 
    &= \parenthese{\mathbf{W} \otimes I_d} \parenthese{\nabla f^{cat}_{\sigma_q(k)} - \nabla f^{cat}_{\sigma_q(k-1)} +\vect{d}^{cat}_{q,k-1}} \nonumber \\
    &= \parenthese{\mathbf{W} \otimes I_d} \parenthese{\nabla f^{cat}_{\sigma_q(k)} - \nabla f^{cat}_{\sigma_q(k-1)}} 
        + \parenthese{\mathbf{W} \otimes I_d}^2 \parenthese{\nabla f^{cat}_{\sigma_q(k-1)} - \nabla f^{cat}_{\sigma_q(k-2)} +\vect{d}^{cat}_{q,k-2}} \nonumber \\
    &= \sum_{\tau = 1}^{k-1} \parenthese{\mathbf{W} \otimes I_d}^{k-\tau} \parenthese{\nabla f^{cat}_{\sigma_q(\tau+1)} - \nabla f^{cat}_{\sigma_q(\tau)}} + \parenthese{\mathbf{W} \otimes I_d}^{k} \nabla f^{cat}_{\sigma_q(1)} \nonumber \\
    & = \sum_{\tau = 1}^{k-1} \parenthese{\mathbf{W} \otimes I_d}^{k-\tau} 
    \parenthese{\nabla f^{cat}_{\sigma_q(\tau+1)} - \nabla f^{cat}_{\sigma_q(\tau)}} 
    + \parenthese{\mathbf{W} \otimes I_d}^{k} \nabla f^{cat}_{\sigma_q(1)} \nonumber \\
        &\quad - \sum_{\tau=1}^{k-1} \parenthese{\nabla F^{cat}_{\sigma_q(\tau+1)} - \nabla F^{cat}_{\sigma_q(\tau)}} - \nabla F^{cat}_{\sigma_q(1)} + \nabla F^{cat}_{\sigma_q(k)} \nonumber \\
    & = \sum_{\tau = 1}^{k-1} \bracket{\parenthese{\mathbf{W}^{k-\tau} - \frac{1}{n}\mathbf{1}_n \mathbf{1}_n^T}\otimes I_d} \parenthese{\nabla f^{cat}_{\sigma_q(\tau+1)} - \nabla f^{cat}_{\sigma_q(\tau)}} \nonumber \\
        & \quad + \bracket{\parenthese{\mathbf{W}^{k} - \frac{1}{n}\mathbf{1}_n \mathbf{1}_n^T}\otimes I_d}
            \nabla f^{cat}_{\sigma_q(1)}
                + \parenthese{\frac{1}{n}\mathbf{1}_n \mathbf{1}_n^T \otimes I_d}\nabla f^{cat}_{\sigma_q(k)}
\end{align}
where the fourth equality holds since $\nabla F^{cat}_{\sigma_q(k)} - \sum_{\tau=1}^{k-1} \parenthese{\nabla F^{cat}_{\sigma_q(\tau+1)} - \nabla F^{cat}_{\sigma_q(\tau)}} - \nabla F^{cat}_{\sigma_q(1)} = 0$. The fifth equality can be deduced using $\nabla F^{cat}_{\sigma_q(k)} = \parenthese{\frac{1}{n}\mathbf{1}_n \mathbf{1}_n^T \otimes I_d} \nabla f^{cat}_{\sigma_q(k)}$ and $\parenthese{\mathbf{W} \otimes I_d}^{k} = \parenthese{\mathbf{W}^k \otimes I_d}$. Recall that $\norm{\mathbf{W}\otimes I_d} = \norm{\mathbf{W}}$. Taking the norm on \cref{eq:bound_dk_step1}, we have
\begin{align}
    \norm{\vect{d}^{cat}_{q,k}}
    \leq 2\sqrt{n}G \sum_{\tau = 1}^{k-1} \lambda_2^{k-\tau} 
            + \sqrt{n}G\parenthese{\lambda_2^k + 1} 
    \leq 2\sqrt{n}G \parenthese{\frac{\lambda_2}{1-\lambda_2} + 1}
\end{align}
where we have used $\norm{\nabla f^{cat}_{\sigma_q(\tau+1)} - \nabla f^{cat}_{\sigma_q(\tau)}} \leq 2\sqrt{n}G$, $\norm{\mathbf{W}^{k} - \frac{1}{n}\mathbf{1}_n \mathbf{1}_n^T} \leq \lambda_2^{k}$ and $\norm{\frac{1}{n}\mathbf{1}_n \mathbf{1}_n^T} \leq 1$ in the first inequality. We have $\forall i \in \bracket{n}$
\begin{align}
\label{eq:bound_sum_useful}
    \norm{\vect{d}^i_{q,k}} \leq \sum_{i=1}^n \norm{\vect{d}^i_{q,k}} \leq \sqrt{n} \parenthese{\sum_{i=1}^n \norm{\vect{d}^i_{q,k}}^2}^{1/2} = \sqrt{n}\norm{\vect{d}^{cat}_{q,k}}
\end{align}
one can obtain the desired result. 
\end{proof}

\begin{lemma}
Under \Cref{assum:stoch_grad} and let $\sigma_1^2 = 4n \bracket{\parenthese{\frac{G+G_0}{\frac{1}{\lambda_2}-1}}^2 + 2\sigma_0^2}$. For $i \in \bracket{n}, k \in \bracket{K}$, the variance of the local stochastic gradient is uniformly bounded.
\begin{align*}
    \E \bracket{ \norm{
        \vect{d}^i_{q,k} - \widetilde{\vect{d}}^i_{q,k}
    }^2} \leq\sigma_1^2
\end{align*}
\end{lemma}

\begin{proof}
We denote $\widetilde{\vect{d}}^{cat}$ the stochastique version of $\vect{d}^{cat}$, following \cref{eq:bound_dk_step1}, we have
\begin{align}
    \widetilde{\vect{d}}^{cat}_{q,k}
     & = \sum_{\tau = 1}^{k-1} 
     \bracket{\parenthese{\mathbf{W}^{k-\tau} - \frac{1}{n}\mathbf{1}_n \mathbf{1}_n^T}\otimes I_d} 
     \parenthese{
        \widetilde{\nabla}f ^{cat}_{\sigma_q(\tau+1)} - \widetilde{\nabla}f ^{cat}_{\sigma_q(\tau)}
    } \nonumber \\
        & \quad + \bracket{\parenthese{\mathbf{W}^{k} - \frac{1}{n}\mathbf{1}_n \mathbf{1}_n^T}\otimes I_d}
            \widetilde{\nabla}f ^{cat}_{\sigma_q(1)}
                + \parenthese{\frac{1}{n}\mathbf{1}_n \mathbf{1}_n^T \otimes I_d}\widetilde{\nabla}f ^{cat}_{\sigma_q(k)}
\end{align}
Then, we have
\begin{align}
\label{eq:surrogate_grad_diff}
    \vect{d}^{cat}_{q,k}  - \widetilde{\vect{d}}^{cat}_{q,k} 
    &= \sum_{\tau = 1}^{k-1} 
    \bracket{\parenthese{\mathbf{W}^{k-\tau} - \frac{1}{n}\mathbf{1}_n \mathbf{1}_n^T}\otimes I_d} 
    \parenthese{
        \nabla f^{cat}_{\sigma_q(\tau+1)}
        - \widetilde{\nabla}f ^{cat}_{\sigma_q(\tau+1)} 
        + \widetilde{\nabla}f ^{cat}_{\sigma_q(\tau)} 
        - \nabla f^{cat}_{\sigma_q(\tau)} 
        } \nonumber \\
        & \quad + \bracket{\parenthese{\mathbf{W}^{k} - \frac{1}{n}\mathbf{1}_n \mathbf{1}_n^T}\otimes I_d}
            \parenthese{
                \nabla f^{cat}_{\sigma_q(1)}
                - \widetilde{\nabla}f ^{cat}_{\sigma_q(1)}
                } \nonumber \\
            & \qquad+ \parenthese{\frac{1}{n}\mathbf{1}_n \mathbf{1}_n^T \otimes I_d}
            \parenthese{
                \nabla f^{cat}_{\sigma_q(k)}
                - \widetilde{\nabla}f ^{cat}_{\sigma_q(k)}
            }
\end{align}
By \Cref{assum:stoch_grad} and Jensen's inequality, we have 
\begin{align}
    \E &\bracket{\norm{\nabla f^{cat}_{\sigma_q(\tau)}- \widetilde{\nabla}f ^{cat}_{\sigma_q(\tau)}}^2}
    = \E \bracket{\sum_{i=1}^n \norm{\nabla f^{i}_{\sigma_q(\tau)} \parenthese{\vect{x}^i_{q,\tau}} - \widetilde{\nabla}f ^{i}_{\sigma_q(\tau)} \parenthese{\vect{x}^i_{q,\tau}}}^2} \nonumber \\
    &\leq \sqrt{\sum_{i=1}^n \E \bracket{\norm{\nabla f^{i}_{\sigma_q(\tau)} \parenthese{\vect{x}^i_{q,\tau}} - \widetilde{\nabla}f ^{i}_{\sigma_q(\tau)} \parenthese{\vect{x}^i_{q,\tau}}}^2}} \leq \sqrt{n} \sigma_0
\end{align}
The second moment of \cref{eq:surrogate_grad_diff} is written as
\begin{align}
\label{eq:bound_stoch_d_proof}
    &\E \bracket{\norm{
        \vect{d}^{cat}_{q,k}  - \widetilde{\vect{d}}^{cat}_{q,k} 
    }^2} \nonumber \\
    \leq & \E \bracket{\parenthese{\sum_{\tau = 1}^{k-1} 
    \norm{\mathbf{W}^{k-\tau} - \frac{1}{n}\mathbf{1}_n \mathbf{1}_n^T} 
    \norm{
        \nabla f^{cat}_{\sigma_q(\tau+1)}
        - \widetilde{\nabla}f ^{cat}_{\sigma_q(\tau+1)} 
        + \widetilde{\nabla}f ^{cat}_{\sigma_q(\tau)} 
        - \nabla f^{cat}_{\sigma_q(\tau)} 
        }}^2} \nonumber \\
        & \quad + \E \bracket{\norm{\parenthese{{\mathbf{W}^{k} - \frac{1}{n}\mathbf{1}_n \mathbf{1}_n^T}\otimes I_d}
            \parenthese{
                \nabla f^{cat}_{\sigma_q(1)}
                - \widetilde{\nabla}f ^{cat}_{\sigma_q(1)}
                } 
            + \parenthese{\frac{1}{n}\mathbf{1}_n \mathbf{1}_n^T \otimes I_d}
            \parenthese{
                \nabla f^{cat}_{\sigma_q(k)}
                - \widetilde{\nabla}f ^{cat}_{\sigma_q(k)}
            }
            }
            ^2} \nonumber \\
    \leq & \E \bracket{\parenthese{\sum_{\tau = 1}^{k-1} 
    \norm{\mathbf{W}^{k-\tau} - \frac{1}{n}\mathbf{1}_n \mathbf{1}_n^T} 
    \norm{
        \nabla f^{cat}_{\sigma_q(\tau+1)}
        - \widetilde{\nabla}f ^{cat}_{\sigma_q(\tau+1)} 
        + \widetilde{\nabla}f ^{cat}_{\sigma_q(\tau)} 
        - \nabla f^{cat}_{\sigma_q(\tau)} 
        }}^2} \nonumber \\
        & \quad 
        + 4\parenthese{\E \bracket{\norm{\mathbf{W}^{k} - \frac{1}{n}\mathbf{1}_n \mathbf{1}_n^T}^2
            \norm{
                \nabla f^{cat}_{\sigma_q(1)}
                - \widetilde{\nabla}f ^{cat}_{\sigma_q(1)}
            }^2} 
            + \E \bracket{\norm{\frac{1}{n}\mathbf{1}_n \mathbf{1}_n^T}^2
            \norm{
                \nabla f^{cat}_{\sigma_q(k)}
                - \widetilde{\nabla}f ^{cat}_{\sigma_q(k)}
            }^2}
        } \nonumber \\
    \leq & 4n\parenthese{G + G_0}^2 \parenthese{\sum_{\tau = 1}^{k-1} \lambda_2^{k-\tau} }^2 
        + 4n\sigma_0^2 \parenthese{\lambda_2^{2k} + 1} \nonumber \\
    \leq & 4n\parenthese{G + G_0}^2\parenthese{\frac{\lambda_2}{1-\lambda_2}}^2 + 4n\sigma_0^2(\lambda_2 + 1) 
    \leq 4n \bracket{\parenthese{\frac{G+G_0}{\frac{1}{\lambda_2}-1}}^2 + 2\sigma_0^2}
\end{align}
where the first inequality holds since $\E \bracket{\nabla f^{cat}_{\sigma_q(\tau+1)}
        - \widetilde{\nabla}f ^{cat}_{\sigma_q(\tau+1)} 
        + \widetilde{\nabla}f ^{cat}_{\sigma_q(\tau)} 
        - \nabla f^{cat}_{\sigma_q(\tau)} } = 0$. 
The second inequality follows the fact that $\norm{a + b}^2 \leq 4\parenthese{\norm{a}^2 + \norm{b}^2}$. The third inequality comes from \Cref{assum:stoch_grad} and the analysis in \Cref{lmm:bound_d}. Finally, one can obtain the desired result by noticing 
$\E \bracket{\norm{\vect{d}^i_{q,k} - \widetilde{\vect{d}}^i_{q,k}}^2} \leq \sum_{i=1}^n \E \bracket{\norm{\vect{d}^i_{q,k} - \widetilde{\vect{d}}^i_{q,k}}^2} = \E \bracket{\norm{\vect{d}^{cat}_{q,k}  - \widetilde{\vect{d}}^{cat}_{q,k}}^2}$
\end{proof}

\setcounter{theorem}{11}
\begin{lemma}[Lemma 6, \cite{Zhang:2019}]
\label[lemma]{lmm:bound_d_var_red}
Under \Cref{assum:assum_d_var_red}, \Cref{lmm:bound_d}, \Cref{lmm:stoch_variance} and setting $\rho_k = \frac{2}{\parenthese{k+3}^{2/3}}$ and $\rho_k = \frac{1.5}{\parenthese{K-k+2}^{2/3}}$ for $k \in \bracket{\frac{K}{2}}$ and $k \in \bracket{\frac{K}{2}+1,K}$ respectively, we have
\begin{align}
    \E \bracket{
        \norm{\hat{\vect{d}}^i_{q,k-1} - \widetilde{\vect{a}}^i_{q,k}}} 
    \leq \begin{dcases}
        \frac{\sqrt{M}}{\parenthese{k+4}^{1/3}} \qquad k \in \bracket{\frac{K}{2}} \\
        \frac{\sqrt{M}}{\parenthese{K-k+1}^{1/3}} \qquad i \in \bracket{\frac{K}{2} + 1,K}
    \end{dcases}
\end{align}
where $M = \max\{M_1, M_2\}$ where $M_1 = \max \{5^{2/3} \parenthese{V_{\vect{d}} + L_0}^2 , M_0 \}$, $M_0 = 4\parenthese{V^2_{\vect{d}} + \sigma^2} + 32\sqrt{2}V_{\vect{d}}$ and $M_2 = 2.55\parenthese{V^2_{\vect{d}} + \sigma^2} + \dfrac{7\sqrt{2}V_{\vect{d}}}{3}$ and $L_0 = \frac{2}{4^{2/3}} \norm{\widetilde{\vect{d}}^i_{q,1}}$
\end{lemma}

\begin{proof}
In order to prove the lemma, we only need to bound $\E \bracket{\norm{\hat{\vect{d}}^i_{q,k-1} - \widetilde{\vect{a}}^i_{q,k}}^2}$, following the decomposition in \cite{Zhang:2019}, we have
\begin{align}
\label{eq:var_red_develop}
    \E &\bracket{\norm{\hat{\vect{d}}^i_{q,k-1} - \widetilde{\vect{a}}^i_{q,k}}^2}
    = \E \bracket{\norm{
        \hat{\vect{d}}^i_{q,k-1} 
        - (1-\rho_k) \widetilde{\vect{a}}^i_{q,k-1} 
        - \rho_k \widetilde{\vect{d}}^i_{q,k}
    }^2
    } \nonumber \\
    & = \rho_k^2 \E \bracket{\norm{\hat{\vect{d}}^i_{q,k-1}-\widetilde{\vect{d}}^i_{q,k})}^2}
    + (1-\rho_k)^2 \E \bracket{\norm{\hat{\vect{d}}^i_{q,k-1} - \hat{\vect{d}}^i_{q,k-2}}^2} \\
    & \quad + (1-\rho_k)^2 
        \E \bracket{\norm{\hat{\vect{d}}^i_{q,k-2} - \widetilde{\vect{a}}^i_{q,k-1} }^2} \nonumber \\
    & \quad + 2\rho_k (1-\rho_k) 
        \E \bracket{\scalarproduct{ 
            \hat{\vect{d}}^i_{q,k-1} - \widetilde{\vect{d}}^i_{q,k}, \hat{\vect{d}}^i_{q,k-1} - \hat{\vect{d}}^i_{q,k-2} 
        }} \nonumber\\
    & \quad + 2\rho_k (1-\rho_k) 
        \E \bracket{\scalarproduct{ 
             \hat{\vect{d}}^i_{q,k-1} - \widetilde{\vect{d}}^i_{q,k},
             \hat{\vect{d}}^i_{q,k-2} - \widetilde{\vect{a}}^i_{q,k-1}
        }} \nonumber\\
    & \quad + 2(1-\rho_k)^2 
        \E \bracket{\scalarproduct{
            \hat{\vect{d}}^i_{q,k-1} - \hat{\vect{d}}^i_{q,k-2},
            \hat{\vect{d}}^i_{q,k-2} - \widetilde{\vect{a}}^i_{q,k-1}
        }}
\end{align}
The first part of the above equation is written as
\begin{align}
    \E &\bracket{
        \norm{
            \hat{\vect{d}}^i_{q,k-1}-\widetilde{\vect{d}}^i_{q,k})
        }^2
    }
    = \E \bracket{\E_{\sigma} \bracket{
        \norm{
            \hat{\vect{d}}^i_{q,k-1}-\widetilde{\vect{d}}^i_{q,k})
        }^2 \bigm\vert \mathcal{F}_{q,k-1}
    }} \nonumber\\
    & =  \E \bracket{\E_{\sigma} \bracket{
        \norm{
            \hat{\vect{d}}^i_{q,k-1}
            - \vect{d}^i_{q,k}
            + \vect{d}^i_{q,k}
            -\widetilde{\vect{d}}^i_{q,k})
        }^2 \bigm\vert \mathcal{F}_{q,k-1}
    }} \nonumber\\
    & \leq \E \bracket{\E_{\sigma} \bracket{
        \norm{
            \hat{\vect{d}}^i_{q,k-1}
            - \vect{d}^i_{q,k}}^2
        + \norm{ 
            \vect{d}^i_{q,k}
            -\widetilde{\vect{d}}^i_{q,k})}^2
        + 2 \langle
                \hat{\vect{d}}^i_{q,k-1} - \vect{d}^i_{q,k},
                \vect{d}^i_{q,k} -\widetilde{\vect{d}}^i_{q,k}
            \rangle
    \bigm\vert \mathcal{F}_{q,k-1} 
    }} \label{eq:bound_d0}
\end{align}
Using the definition of $\hat{\vect{d}}^i_{q,k-1}$, \Cref{lmm:bound_d} and \Cref{lmm:stoch_variance} and law of total expectation, we have
\begin{align}
    \E \bracket{\E_{\sigma} \bracket{
        \norm{
            \hat{\vect{d}}^i_{q,k-1}
            - \vect{d}^i_{q,k}}^2
        \bigm\vert \mathcal{F}_{q,k-1} 
    }} 
    = \E \bracket{\Var_{\sigma}\left(
        \vect{d}^i_{q,k} \bigm\vert \mathcal{F}_{q,k-1} 
    \right)
    } 
    \leq \E \bracket{\E_{\sigma} \bracket{
        \norm{\vect{d}^i_{q,k}}^2 \bigm\vert \mathcal{F}_{q,k-1} 
    }} 
    \leq V_{\vect{d}}^2 \label{eq:bound_d1} 
\end{align}
\begin{align}
    \E \bracket{\E_{\sigma} \bracket{ \norm{
        \vect{d}^i_{q,k} -\widetilde{\vect{d}}^i_{q,k})
    }^2 
    }\vert \mathcal{F}_{q,k-1}}
    \leq \sigma^2_1 \label{eq:bound_d2}
\end{align}
Recall that $\mathcal{H}_{q,k}$ is the filtration related to the randomness of $\vect{\widetilde{d}}^i_{q,k}$ and $\hat{\vect{d}}^i_{q,k-1}$ and $\vect{d}^i_{q,k}$ is $\mathcal{F}_{q,k} $-measurable, then one can write
\begin{align}
    &\E \bracket{\E_{\sigma} \bracket{
        \scalarproduct{
                \hat{\vect{d}}^i_{q,k-1} - \vect{d}^i_{q,k},
                \vect{d}^i_{q,k} -\widetilde{\vect{d}}^i_{q,k}
        }
        \bigm\vert \mathcal{F}_{q,k-1}
    }} \nonumber \\
    =& \E \bracket{
        \scalarproduct{
                \hat{\vect{d}}^i_{q,k-1} - \vect{d}^i_{q,k},
                \vect{d}^i_{q,k} -\widetilde{\vect{d}}^i_{q,k}
        }
    } \nonumber \\
    =& \E \bracket{\E_{\sigma} \bracket{
        \scalarproduct{
                \hat{\vect{d}}^i_{q,k-1} - \vect{d}^i_{q,k},
                \vect{d}^i_{q,k} -\vect{\widetilde{d}}^i_{q,k}
        }
        \bigm\vert \mathcal{F}_{q,k}
    }} \nonumber \\
    =& \E \bracket{
        \scalarproduct{
                \hat{\vect{d}}^i_{q,k-1} - \vect{d}^i_{q,k},
                \E_{\sigma} \bracket{ \vect{d}^i_{q,k} -\vect{\widetilde{d}}^i_{q,k}
        \bigm\vert \mathcal{F}_{q,k}
    }}
    } \tag{by $\mathcal{F}_{q,k}$-measurability} \nonumber\\
    =& \E \bracket{ \E_{\vect{\widetilde{d}}} \bracket{
        \scalarproduct{
                \hat{\vect{d}}^i_{q,k-1} - \vect{d}^i_{q,k},
                \E_{\sigma} \bracket{ \vect{d}^i_{q,k} -\vect{\widetilde{d}}^i_{q,k}
        \bigm\vert \mathcal{F}_{q,k}
    }}
    }\bigm\vert \mathcal{H}_{q,k-1}} \nonumber \\
    =& \E \bracket{
        \scalarproduct{
                \hat{\vect{d}}^i_{q,k-1} - \vect{d}^i_{q,k}, \E_{\vect{\widetilde{d}}} \bracket{
                \E_{\sigma} \bracket{ \vect{d}^i_{q,k} -\vect{\widetilde{d}}^i_{q,k}
        \bigm\vert \mathcal{F}_{q,k}
    } \bigm\vert \mathcal{H}_{q,k-1}}
    }} \nonumber \\
    =& \E \bracket{
        \scalarproduct{
                \hat{\vect{d}}^i_{q,k-1} - \vect{d}^i_{q,k}, \E_{\sigma} \bracket{
                \E_{\vect{\widetilde{d}}} \bracket{ \vect{d}^i_{q,k} -\vect{\widetilde{d}}^i_{q,k}
        \bigm\vert \mathcal{H}_{q,k-1}
    } \bigm\vert \mathcal{F}_{q,k}}
    }} \nonumber \tag{by Fubini's theorem} \\
    =& 0 \label{eq:bound_d3}
\end{align}
where the last equation holds since $\E_{\vect{\widetilde{d}}}\bracket{\vect{\widetilde{d}}^i_{q,k}\bigm\vert \mathcal{H}_{q,k-1}} = \vect{d}^i_{q,k}$. Combining \cref{eq:bound_d1,eq:bound_d2,eq:bound_d3}, \cref{eq:bound_d0} is upper bounded by
\begin{align}
    \label{eq:bound_d_hat_tilde}
    \E &\bracket{
        \norm{
            \hat{\vect{d}}^i_{q,k-1}-\vect{\widetilde{d}}^i_{q,k})
        }^2
    } \leq V_{\vect{d}}^2 + \sigma^2_1 \triangleq V
\end{align}
We are now bounding $\E \bracket{\norm{\hat{\vect{d}}^i_{q,k-1} - \hat{\vect{d}}^i_{q,k-2}}^2} $, using the definition of $\hat{\vect{d}}^i_{q,k}$ and \Cref{lmm:bound_d}. We have
\begin{align}
    \label{eq:bound_d_hat_diff}
    &\E \bracket{\norm{
        \hat{\vect{d}}^i_{q,k-1} - \hat{\vect{d}}^i_{q,k-2} 
    }^2} \nonumber \\
    =& \E \bracket{\E_{\sigma} \bracket{
        \norm{
            \hat{\vect{d}}^i_{q,k-1} - \hat{\vect{d}}^i_{q,k-2} 
        }^2 
        \bigm\vert \mathcal{F}_{q,k-2}
    }} \nonumber \\
    =& \E \bracket{ \E_{\sigma} \bracket{
        \norm{
            \frac{\sum_{\ell=k}^K \vect{d}^i_{q,\ell}}{K-k+1} 
            - \frac{\sum_{\ell=k-1}^K \vect{d}^i_{q,\ell}}{K-k+2}
        }^2  
        \bigm\vert \mathcal{F}_{q,k-2}
    }} \nonumber \\
    =& \E \bracket{ \E_{\sigma} \bracket{
        \norm{
            \frac{\sum_{\ell=k}^K \vect{d}^i_{q,\ell}}{K-k+1} 
            - \frac{\sum_{\ell=k}^K \vect{d}^i_{q,\ell}}{K-k+2}
            - \frac{\vect{d}^i_{q,k-1}}{K-k+2}
        }^2  
        \bigm\vert \mathcal{F}_{q,k-2}
    }} \nonumber \\
    =& \E \bracket{ \E_{\sigma} \bracket{
        \norm{
            \frac{\sum_{\ell=k}^K \vect{d}^i_{q,\ell}}{(K-k+1)(K-k+2)} 
            - \frac{\vect{d}^i_{q,k-1}}{K-k+2}
        }^2  
        \bigm\vert \mathcal{F}_{q,k-2}
    }} \nonumber \\
    \leq& \E \bracket{ \E_{\sigma} \bracket{
        \left(
            \frac{\sum_{\ell=k}^K \norm{\vect{d}^i_{q,\ell}}}{(K-k+1)(K-k+2)} 
            + \frac{\norm{\vect{d}^i_{q,k-1}}}{K-k+2}
        \right)^2  
        \bigm\vert \mathcal{F}_{q,k-2}
    }} \nonumber \\
    \leq& \frac{4V^2_{\vect{d}}}{\left(K-k+2\right)^2}
    \triangleq \frac{L}{\left(K-k+2\right)^2}
\end{align}
More over, we have
\begin{align}
\label{eq:bound_d4}
    &\E \bracket{\scalarproduct{
            \hat{\vect{d}}^i_{q,k-1} - \widetilde{\vect{d}}^i_{q,k}, \hat{\vect{d}}^i_{q,k-1} - \hat{\vect{d}}^i_{q,k-2} 
        }} \nonumber\\
    =& \E \bracket{\E_{\sigma, \vect{\widetilde{d}}} \bracket{
        \scalarproduct{
            \hat{\vect{d}}^i_{q,k-1} - \widetilde{\vect{d}}^i_{q,k}, \hat{\vect{d}}^i_{q,k-1} - \hat{\vect{d}}^i_{q,k-2} 
        } \bigm\vert \mathcal{F}_{q,k-1}, \mathcal{H}_{q,k-1}
        }} \nonumber\\
    =& \E \bracket{\scalarproduct{
            \E_{\sigma, \vect{\widetilde{d}}} \bracket{
            \hat{\vect{d}}^i_{q,k-1} - \widetilde{\vect{d}}^i_{q,k} \bigm\vert \mathcal{F}_{q,k-1}, \mathcal{H}_{q,k-1}}, 
            \hat{\vect{d}}^i_{q,k-1} - \hat{\vect{d}}^i_{q,k-2} 
        }} \nonumber\\
    =& 0
\end{align}
since $\E_{\widetilde{\vect{d}}} \bracket{\vect{\widetilde{d}}^i_{q,k} \bigm\vert \mathcal{H}_{q,k-1}} = \vect{d}^i_{q,k}$ and $\E_{\sigma} \bracket{\vect{d}^i_{q,k}  \bigm\vert \mathcal{F}_{q,k-1} } = \hat{\vect{d}}^i_{q,k}$. Using the same argument, we can deduce
\begin{align}
\label{eq:bound_d5}
    &\E \bracket{\scalarproduct{ 
             \hat{\vect{d}}^i_{q,k-1} - \widetilde{\vect{d}}^i_{q,k},
             \hat{\vect{d}}^i_{q,k-2} - \widetilde{\vect{a}}^i_{q,k-1}
        }} \nonumber\\
    =& \E \bracket{\E_{\sigma, \vect{\widetilde{d}}} \bracket{
            \scalarproduct{
             \hat{\vect{d}}^i_{q,k-1} - \widetilde{\vect{d}}^i_{q,k},
             \hat{\vect{d}}^i_{q,k-2} - \widetilde{\vect{a}}^i_{q,k-1}
        } \bigm\vert \mathcal{F}_{q,k-1}, \mathcal{H}_{q,k-1}, 
        }} \nonumber\\
    =& \E \bracket{
            \scalarproduct{ \E_{\sigma, \vect{\widetilde{d}}} \bracket{
             \hat{\vect{d}}^i_{q,k-1} - \widetilde{\vect{d}}^i_{q,k} \bigm\vert \mathcal{F}_{q,k-1}, \mathcal{H}_{q,k-1}},
             \hat{\vect{d}}^i_{q,k-2} - \widetilde{\vect{a}}^i_{q,k-1}
        }} \nonumber\\
    =& 0
\end{align}
where we have use law of total expectation and conditional unbiasness of $\widetilde{\vect{d}}^{i}_{q,k}$. 
Using Young's inequality and \cref{eq:bound_d_hat_diff}, one can write 
\begin{align}
    \label{eq:bound_d6}
    \E &\bracket{\scalarproduct{
            \hat{\vect{d}}^i_{q,k-1} - \hat{\vect{d}}^i_{q,k-2},
            \hat{\vect{d}}^i_{q,k-2} - \widetilde{\vect{a}}^i_{q,k-1}
        }} \nonumber\\
    & \quad \leq  \E \bracket{
        \frac{1}{2\alpha_k} \norm{\hat{\vect{d}}^i_{q,k-1} - \hat{\vect{d}}^i_{q,k-2}}^2
        + \frac{\alpha_k}{2} \norm{\hat{\vect{d}}^i_{q,k-2} - \widetilde{\vect{a}}^i_{q,k-1}}^2
    } \nonumber \\
    & \quad \leq \frac{L}{2\alpha_k (K-k+2)^2} 
        + \frac{\alpha_k}{2} \E \bracket{\norm{\hat{\vect{d}}^i_{q,k-2} - \widetilde{\vect{a}}^i_{q,k-1}}^2}
\end{align}
With the above analysis, we can deduce that
\begin{align}
    \label{eq:bound_d_atilde_recur}
    \E \bracket{\norm{\hat{\vect{d}}^i_{q,k-1} - \widetilde{\vect{a}}^i_{q,k}}^2} 
    &\leq \rho_k^2 V
         + (1-\rho_k)^2 \frac{L}{(K-k+2)^2} 
         + (1-\rho_k)^2 \E \bracket{
            \norm{\hat{\vect{d}}^i_{q,k-2} - \widetilde{\vect{a}}^i_{q,k-1}}^2
            } \nonumber\\ 
         & \quad + (1-\rho_k)^2 \left(
            \frac{L}{\alpha_k (K-k+2)^2} 
            + \alpha_k \E \bracket{\norm{\hat{\vect{d}}^i_{q,k-2} - \widetilde{\vect{a}}^i_{q,k-1}}^2}
        \right)
\end{align}
Setting $\alpha_k = \frac{\rho_k}{2}$, we have
\begin{align}
    \label{eq:recurrent_psi}
    \E \bracket{\norm{\hat{\vect{d}}^i_{q,k-1} - \widetilde{\vect{a}}^i_{q,k}}^2} 
    & \leq \rho_k^2 V 
            + \parenthese{1-\rho_k}^2 \parenthese{1+\frac{2}{\rho_k}} \frac{L}{\parenthese{K-k+2}^2} \nonumber \\
            & \quad + \parenthese{1-\rho_k}^2 \parenthese{1+\frac{\rho_k}{2}} \E \bracket{\norm{\hat{\vect{d}}^i_{q,k-2} - \widetilde{\vect{a}}^i_{q,k-1}}^2} \nonumber \\
    & \leq \rho_k^2 V 
        + \parenthese{1+\frac{2}{\rho_k}} \frac{L}{\parenthese{K-k+2}^2} 
        + \parenthese{1-\rho_k} \E \bracket{\norm{\hat{\vect{d}}^i_{q,k-2} - \widetilde{\vect{a}}^i_{q,k-1}}^2}
\end{align}
For $k \leq \frac{K}{2}+1$, we set $\rho_k = \frac{2}{\parenthese{k+3}^{2/3}}$ and recall that $K-k+2 \geq k$, \cref{eq:recurrent_psi} is written as:
\begin{align}
    \E& \bracket{\norm{\hat{\vect{d}}^i_{q,k-1} - \widetilde{\vect{a}}^i_{q,k}}^2} \nonumber \\
    &\leq \frac{4}{\parenthese{k+3}^{4/3}}V 
        + \parenthese{1 + \parenthese{k+3}^{2/3}} \frac{L}{k^2}
        + \parenthese{1 - \frac{2}{\parenthese{k+3}^{2/3}}} \E \bracket{\norm{\hat{\vect{d}}^i_{q,k-2} - \widetilde{\vect{a}}^i_{q,k-1}}^2} \nonumber \\
    & \leq \frac{4}{\parenthese{k+3}^{4/3}}V 
        + \parenthese{1 + \parenthese{k+3}^{2/3}} \frac{16L}{\parenthese{k+3}^2}
        + \parenthese{1 - \frac{2}{\parenthese{k+3}^{2/3}}} \E \bracket{\norm{\hat{\vect{d}}^i_{q,k-2} - \widetilde{\vect{a}}^i_{q,k-1}}^2} \nonumber \\
    & \leq \frac{4}{\parenthese{k+3}^{4/3}}V 
        + \frac{16L}{\parenthese{k+3}^{4/3}}
        +\frac{16L}{\parenthese{k+3}^{4/3}}
        + \parenthese{1 - \frac{2}{\parenthese{k+3}^{2/3}}} \E \bracket{\norm{\hat{\vect{d}}^i_{q,k-2} - \widetilde{\vect{a}}^i_{q,k-1}}^2} \nonumber \\
    & \leq \frac{4V + 32L}{\parenthese{k+3}^{4/3}} 
        + \parenthese{1 - \frac{2}{\parenthese{k+3}^{2/3}}} \E \bracket{\norm{\hat{\vect{d}}^i_{q,k-2} - \widetilde{\vect{a}}^i_{q,k-1}}^2} \nonumber \\
    & \triangleq \frac{M_0}{\parenthese{k+3}^{4/3}}
        + \parenthese{1 - \frac{2}{\parenthese{k+3}^{2/3}}} \E \bracket{\norm{\hat{\vect{d}}^i_{q,k-2} - \widetilde{\vect{a}}^i_{q,k-1}}^2} 
\end{align}
We consider the base step where $k=1$,
\begin{align}
    \E \bracket{\norm{\hat{\vect{d}}^i_{q,0} - \widetilde{\vect{a}}^i_{q,1}}^2} 
    &= \E \bracket{\norm{\frac{1}{K} \sum_{\ell=1}^K \vect{d}^i_{q,\ell} 
        - \frac{2}{4^{2/3}} \widetilde{\vect{d}}^i_{q,1}}^2} \nonumber \\
    & \leq \parenthese{
        V_{\vect{d}} + \frac{2}{4^{2/3}} \norm{\widetilde{\vect{d}}^i_{q,1}}
    }^2 \nonumber \\
    & \leq \parenthese{
        V_{\vect{d}} + \frac{2}{4^{2/3}} G_0
    }^2 \nonumber \\
    & \triangleq \parenthese{
        V_{\vect{d}} + L_0
    }^2
\end{align}
Set $M_1 = \max \curlybracket{5^{2/3} \parenthese{V_{\vect{d}} + L_0}^2 , M_0}$. For $k \in \bracket{\frac{K}{2}+1}$, we claim that $\E \bracket{\norm{\hat{\vect{d}}^i_{q,k-1} - \widetilde{\vect{a}}^i_{q,k}}^2} \leq \dfrac{M_1}{(k+4)^{2/3}}$. Suppose the claim holds for $k-1$, we have
\begin{align}
    \label{eq:bound_d_k_small}
    \E \bracket{\norm{\hat{\vect{d}}^i_{q,k-1} - \widetilde{\vect{a}}^i_{q,k}}^2}
    & \leq \frac{M_1}{(k+3)^{4/3}} + \E \bracket{\norm{\hat{\vect{d}}^i_{q,k-2} - \widetilde{\vect{a}}^i_{q,k-1}}^2}
        \parenthese{1 - \frac{2}{\parenthese{k+3}^{2/3}}} \nonumber \\
    & \leq \frac{M_1}{(k+3)^{4/3}} + \frac{M_1}{(k+3)^{2/3}} \cdot
        \frac{\parenthese{k+3}^{2/3} - 2}{\parenthese{k+3}^{2/3}} \nonumber \\
    & \leq \frac{M_1 \parenthese{(k+3)^{2/3} - 1}}{(k+3)^{4/3}} \nonumber \\
    & \leq \frac{M_1}{(k+4)^{2/3}}
\end{align}
since $\dfrac{(k+3)^{2/3} - 1}{k+3)^{4/3}} \leq \dfrac{1}{(k+4)^{2/3}}$. For $k \in \bracket{\frac{K}{2}+1,K}$, we set $\rho_k = \dfrac{1.5}{(K-k+2)^{2/3}}$, thus
\begin{align}
    \E \bracket{\norm{\hat{\vect{d}}^i_{q,k-1} - \widetilde{\vect{a}}^i_{q,k}}^2}
    & \leq \frac{2.55 V}{(K-k+2)^{4/3}} 
        + \parenthese{1 + \frac{4}{3} \parenthese{K-k+2}^{2/3}} \frac{L}{\parenthese{K-k+2}^2} \nonumber \\
        & \quad + \E \bracket{\norm{\hat{\vect{d}}^i_{q,k-2} - \widetilde{\vect{a}}^i_{q,k-1}}^2} \parenthese{1 - \frac{1.5}{\parenthese{K-k+2}^{2/3}}} \nonumber \\
    & \leq \frac{2.55 V}{(K-k+2)^{4/3}} 
        + \frac{L}{\parenthese{K-k+2}^{4/3}} 
        + \frac{4}{3} \frac{L}{\parenthese{K-k+2}^{4/3}} \nonumber \\
        & \quad + \E \bracket{\norm{\hat{\vect{d}}^i_{q,k-2} - \widetilde{\vect{a}}^i_{q,k-1}}^2} \parenthese{1 - \frac{1.5}{\parenthese{K-k+2}^{2/3}}} \nonumber \\
    & \leq \frac{2.55 V + 7L/3}{(K-k+2)^{4/3}} 
         + \E \bracket{\norm{\hat{\vect{d}}^i_{q,k-2} - \widetilde{\vect{a}}^i_{q,k-1}}^2} \parenthese{1 - \frac{1.5}{\parenthese{K-k+2}^{2/3}}} \nonumber \\
    & \triangleq \frac{M_2}{(K-k+2)^{4/3}} 
         + \E \bracket{\norm{\hat{\vect{d}}^i_{q,k-2} - \widetilde{\vect{a}}^i_{q,k-1}}^2} \parenthese{1 - \frac{1.5}{\parenthese{K-k+2}^{2/3}}} 
\end{align}
Let $M = \max \curlybracket{M_1, M_2}$ and  $k \in \bracket{\frac{K}{2} + 1,K}$, we claim that $\E \bracket{\norm{\hat{\vect{d}}^i_{q,k-1} - \widetilde{\vect{a}}^i_{q,k}}^2} \leq \dfrac{M}{\parenthese{K-k+1}^{2/3}}$. The base step is verified by \cref{eq:bound_d_k_small}. We now suppose the claim holds for $k-1$, let's prove for $k$.
\begin{align}
    \E \bracket{\norm{\hat{\vect{d}}^i_{q,k-1} - \widetilde{\vect{a}}^i_{q,k}}^2}
    \label{eq:bound_d_k_big}
    & \leq \frac{M}{\parenthese{K-k+2}^{4/3}} + \frac{M}{\parenthese{K-k+2}^{2/3}} \cdot \frac{\parenthese{K-k+2}^{2/3} - 1.5}{\parenthese{K-k+2}^{2/3}} \nonumber\\
    & = \frac{M \parenthese{\parenthese{K-k+2}^{2/3} - 0.5}}{\parenthese{K-k+2}^{4/3}} \nonumber\\
    & \leq \frac{M}{\parenthese{K-k+1}^{2/3}}
\end{align}
Thus, from \cref{eq:bound_d_k_small} and \cref{eq:bound_d_k_big}, we have 
\begin{align}
    \E \bracket{\norm{\hat{\vect{d}}^i_{q,k-1} - \widetilde{\vect{a}}^i_{q,k}}^2}
    \leq \begin{dcases}
        \frac{M}{\parenthese{k+4}^{2/3}} \qquad  k \in \bracket{1,\frac{K}{2}} \\
        \frac{M}{\parenthese{K-k+1}^{2/3}} \qquad k \in \bracket{\frac{K}{2}+1,K}
    \end{dcases}
\end{align}
Thus, using Jensen inequality, we have
\begin{align}
    \E \bracket{\norm{\hat{\vect{d}}^i_{q,k-1} - \widetilde{\vect{a}}^i_{q,k}}}
    &\leq \sqrt{\E \bracket{\norm{\hat{\vect{d}}^i_{q,k-1} - \widetilde{\vect{a}}^i_{q,k}}^2}} \nonumber \\
    &= \sqrt{\E \bracket{\norm{\hat{\vect{d}}^i_{q,k-1} - \widetilde{\vect{a}}^i_{q,k}}^2}} \nonumber \\
    & \leq \begin{dcases}
        \frac{\sqrt{M}}{\parenthese{k+4}^{1/3}} \qquad  k \in \bracket{1,\frac{K}{2}} \\
        \frac{\sqrt{M}}{\parenthese{K-k+1}^{1/3}} \qquad  k \in \bracket{\frac{K}{2}+1,K}
    \end{dcases}
\end{align}

\end{proof}


\begin{claim}
\label{clm:f_hat_f_bar}
\begin{align}
    \label{eq:F_bar_hat}
    \E \bracket{\norm{\nabla \overline{F}_{q,k-1}(\overline{\vect{x}}_{q,k}) -  \nabla \hat{F}_{q,k-1}}}
    \leq \beta D 
\end{align}
\end{claim}

\begin{claimproof}
Recall the definition of $\overline{F}_{q,k-1}$ and $\hat{F}_{q,k-1}$,
\begin{align}
    \overline{F}_{q,k-1} (\overline{\vect{x}}_{q,k}) = \frac{1}{K-k+1} \sum_{\ell=k}^K \frac{1}{ n}\sum_{i=1}^n f^i_{\sigma_q (\ell)} (\overline{\vect{x}}_{q,k}) \nonumber\\
    \hat{F}_{q,k-1} = \frac{1}{K-k+1} \sum_{\ell=k}^K \frac{1}{ n}\sum_{i=1}^n f^i_{\sigma_q (\ell)} (\vect{x}^i_{q,\ell}) \nonumber
\end{align}
we have,
\begin{align}
    \E &\bracket{\norm{\nabla \overline{F}_{q,k-1}(\overline{\vect{x}}_q^k) -  \nabla \hat{F}_{q,k-1}}} \nonumber \\
    &= \E \bracket{\norm{
        \frac{1}{K-k+1} \cdot \frac{1}{n} \sum_{\ell=k}^K \sum_{i=1}^n
        \parenthese{
            \nabla f^i_{\sigma_{q} (\ell)} \parenthese{\overline{\vect{x}}_{q,k}}
            -  \nabla f^i_{\sigma_{q} (\ell)} \parenthese{\vect{x}^i_{q,\ell}}
            }
    }} \nonumber \\
    & \leq \E \bracket{
        \frac{1}{K-k+1} \cdot \frac{1}{n} \sum_{\ell=k}^K \sum_{i=1}^n
        \norm{
            \nabla f^i_{\sigma_{q} (\ell)} \parenthese{\overline{\vect{x}}_{q,k}}
            -  \nabla f^i_{\sigma_{q} (\ell)} \parenthese{\vect{x}^i_{q,\ell}}
            }
    } \nonumber \\
    & \leq \E \bracket{
        \frac{1}{K-k+1} \cdot \frac{1}{n} \sum_{\ell=k}^K \sum_{i=1}^n
        \beta \norm{
            \overline{\vect{x}}_{q,k} -  \vect{x}^i_{q,\ell}
            }
    } \nonumber \tag{by $\beta$-smoothness} \\
    & \leq \beta D 
\end{align}
\end{claimproof}

\begin{claim}
\label{clm:bound_F_K}
\begin{equation}
    \sum_{k=1}^K \E \bracket{\norm{\nabla \overline{F}_{q,k-1}(\overline{\vect{x}}_{q,k}) - \widetilde{\vect{a}}_{q,k}^i}} \leq \beta D +  \parenthese{N + \sqrt{M}} 3K^{2/3}
\end{equation}
\end{claim}

\begin{claimproof}
\begin{align}
    \label{eq:bound_with_K}
    \sum_{k=1}^K & \E \bracket{\norm{\nabla \overline{F}_{q,k-1}(\overline{\vect{x}}_{q,k}) - \vect{\Tilde{a}}_{q,k}^i}} \nonumber \\
    & \leq \sum_{k=1}^K \E \bracket{ \norm{\nabla \overline{F}_{q,k-1}(\overline{\vect{x}}_{q,k}) -  \nabla \hat{F}_{q,k-1}}}
        + \sum_{k=1}^K \E \bracket{\norm{\nabla \hat{F}_{q,k-1} - \vect{\widetilde{a}}_{q,k}^i }} \nonumber \\
    & \leq \beta D
        + \sum_{k=1}^K \E \bracket{\norm{\nabla \hat{F}_{q,k-1} - \hat{\vect{d}}^i_{q,k-1}}}
        + \sum_{k=1}^K \E \bracket{\norm{ \hat{\vect{d}}^i_{q,k-1} - \vect{\widetilde{a}}_{q,k}^i}}
\end{align}
where we have used \Cref{clm:f_hat_f_bar} and triangle inequality in the last inequality. Using \Cref{lmm:bound_d_avg_consensus}, we have
\begin{align}
    \label{eq:bound_value_d_avg_consensus}
    \sum_{k=1}^K &\E \bracket{\norm{\nabla \hat{F}_{q,k-1} - \hat{\vect{d}}^i_{q,k-1}}} \nonumber \\
    &= \sum_{k=1}^{K/2} \E \bracket{\norm{\nabla \hat{F}_{q,k-1} - \hat{\vect{d}}^i_{q,k-1}}}
        + \sum_{k=K/2 + 1}^{K}\E\bracket{\norm{\nabla \hat{F}_{q,k-1} - \hat{\vect{d}}^i_{q,k-1}}} \nonumber \\
    & \leq \sum_{k=1}^{K/2}\frac{N}{k} + \sum_{k=K/2 + 1}^{K}\frac{N}{K-k+1}
\end{align}
By \Cref{lmm:bound_d_var_red}, we also have
\begin{align}
    \label{eq:bound_value_d_var_red}
    \sum_{k=1}^K &\E \bracket{\norm{ \hat{\vect{d}}^i_{q,k-1} - \vect{\widetilde{a}}_{q,k}^i}} \nonumber \\
    &= \sum_{k=1}^{K/2}\E \bracket{\norm{ \hat{\vect{d}}^i_{q,k-1} - \vect{\widetilde{a}}_{q,k}^i}}
        + \sum_{k=K/2 + 1}^{K}\E \bracket{\norm{ \hat{\vect{d}}^i_{q,k-1} - \vect{\widetilde{a}}_{q,k}^i}} \nonumber \\
    & \leq \sum_{k=1}^{K/2} \frac{\sqrt{M}}{\parenthese{k+4}^{1/3}} 
        + \sum_{k=K/2 + 1}^{K}\frac{\sqrt{M}}{\parenthese{K-k+1}^{1/3}} 
\end{align}
Combining \cref{eq:bound_value_d_avg_consensus} and \cref{eq:bound_value_d_var_red}, \cref{eq:bound_with_K} is written as 
\begin{align}
    \sum_{k=1}^K & \E \bracket{\norm{\nabla \overline{F}_{q,k-1}(\overline{\vect{x}}_q^k) - \vect{\Tilde{a}}_{q,k}^i}} \nonumber \\
    &\leq \beta D 
        +  \sum_{k=1}^{K/2} \parenthese{\frac{N}{k} + \frac{\sqrt{M}}{\parenthese{k+4}^{1/3}}}
        + \sum_{k=K/2 + 1}^{K} \parenthese{\frac{N}{K-k+1} + \frac{\sqrt{M}}{\parenthese{K-k+1}^{1/3}}} \nonumber \\
    & \leq  \beta D
        + \parenthese{N + \sqrt{M}} \sum_{k=1}^{K/2} \frac{1}{\parenthese{k+4}^{1/3}}
        + \parenthese{N + \sqrt{M}} \sum_{k=K/2 + 1}^{K} \frac{1}{\parenthese{K-k+1}^{1/3}} \nonumber \\
    & \leq \beta D
        + \parenthese{N + \sqrt{M}} \sum_{k=1}^{K/2} \frac{1}{k^{1/3}}
        + \parenthese{N + \sqrt{M}} \sum_{l=1}^{K/2} \frac{1}{l^{1/3}} \nonumber \\
    & \leq  \beta D
        +  2\parenthese{N + \sqrt{M}} \int_{0}^{K/2} \frac{1}{s^{1/3}}ds \nonumber \\
    & \leq \beta D +  2\parenthese{N + \sqrt{M}} \frac{3}{2}\parenthese{\frac{K}{2}}^{2/3} \nonumber \\
    & \leq \beta D +  \parenthese{N + \sqrt{M}} 3K^{2/3}
\end{align}
\end{claimproof}

\subsection{Proof of \Cref{thm:convex}}
\label{chap:convex_analysis}

\setcounter{theorem}{2}
\begin{theorem}
Given a convex set $\mathcal{K}$ with diameters $D$. Assume that function $F_t$ are convex, \emph{$\beta$-smooth} and $\norm{\nabla F_t} \leq G$ for $t \in \bracket{T}$. Setting $Q=T^{2/5}, K=T^{3/5}, T=QK$ and step-size $\eta_k = \frac{1}{k}$. Let $\rho_k = \frac{2}{\parenthese{k+3}^{2/3}}$ and $\rho_k = \frac{1.5}{\parenthese{K-k+2}^{2/3}}$ when $k \in \bracket{1,\frac{K}{2}}$ and $k \in \bracket{\frac{K}{2} +1,K}$ respectively. Then, the expected regret of \Cref{algo:online-dist-FW} is at most
\begin{align}
    \E \bracket{\mathcal{R}_T}
    \leq \parenthese{GD + 2\beta D^2}T^{2/5} 
        + \parenthese{C + 6D\parenthese{N + \sqrt{M}}}T^{4/5} 
        + \frac{3}{5}\beta D^2 T^{2/5}\log (T)
\end{align}
where $N = k_0 \cdot nG\max \{\lambda_2 \parenthese{1 + \frac{2}{1-\lambda_2}}, 2\}$
and $M = \max\{M_1, M_2\}$ where 
$    M_0 = 4 \parenthese{V^2_{\vect{d}} + \sigma_1^2} + 128 V^2_{\vect{d}}$,
$    M_1 = \max \curlybracket{5^{2/3} \parenthese{V_{\vect{d}} + \frac{2}{4^{2/3}} G_0}^2 , M_0}$
and
$    M_2 = 2.55\parenthese{V^2_{\vect{d}} + \sigma^2_1} + \dfrac{28 V^2_{\vect{d}}}{3}$. All the constant are defined in \Cref{lmm:bound_d}, \Cref{lmm:stoch_variance}, \Cref{lmm:bound_d_avg_consensus} and \Cref{lmm:bound_d_var_red}.
\end{theorem}

\begin{proof}
    \begin{align}
    \label{eq:main_lemma}
        \E &\bracket{
            \Bar{F}_{q,k-1}(\overline{\vect{x}}_{q,k+1}) - \Bar{F}_{q,k-1}(\vect{x}^*)
        } \nonumber \\
        & \leq (1-\eta_k) \E \bracket{
            \Bar{F}_{q,k-1}(\overline{\vect{x}}_{q,k}) - \Bar{F}_{q,k-1}(\vect{x}^*)
        }
        + \frac{\eta_k}{n} \sum_{i=1}^n \E \bracket{
            \scalarproduct{\widetilde{\vect{a}}^i_{q,k}, \vect{v}^i_{q,k} - \vect{x}^*}
        } \nonumber \\
        & \quad + \frac{\eta_k}{n} D \sum_{i=1}^n \E \bracket{
            \norm{
                \nabla \Bar{F}_{q,k-1} (\overline{\vect{x}}_{q,k}) - \widetilde{\vect{a}}^i_{q,k}
            }
        }
        + \frac{\beta}{2} \eta_k^2 D^2 
    \end{align}
As 
$\E \bracket{\Bar{F}_{q,k-1}(\overline{\vect{x}}_{q,k}) - \Bar{F}_{q,k-1}(\vect{x}^*)} = \E \bracket{\Bar{F}_{q,k-2}(\overline{\vect{x}}_{q,k}) - \Bar{F}_{q,k-2}(\vect{x}^*)}$, we can apply \cref{eq:main_lemma} recursively for $k \in \{1, \ldots, K\}$, thus

\begin{align}
        \E &\bracket{
            \Bar{F}_{q,0}(\overline{\vect{x}}_{q}) - \Bar{F}_{q,0}(\vect{x}^*)
        } \nonumber \\
        & \leq \prod_{k=1}^K (1-\eta_k) \E \bracket{
            \Bar{F}_{q,0}(\overline{\vect{x}}_{q,1}) - \Bar{F}_{q,0}(\vect{x}^*)
        }
        + \sum_{k=1}^K \prod_{k'=k+1}^K (1-\eta_{k'}) \frac{\eta_k}{n} \sum_{i=1}^n \E \bracket{
            \langle \widetilde{\vect{a}}^i_{q,k}, \vect{v}^i_{q,k} - \vect{x}^*\rangle
        } \nonumber \\
        & \qquad + \sum_{k=1}^K \prod_{k'=k+1}^K (1-\eta_{k'}) \frac{\eta_k}{n} D \sum_{i=1}^n \E \bracket{
            \norm{
                \nabla \Bar{F}_{q,k-1} (\overline{\vect{x}}_{q,k}) - \widetilde{\vect{a}}^i_{q,k}
            }
        }
        + \frac{\beta}{2}D^2 \sum_{k=1}^K \prod_{k'=k+1}^K (1-\eta_{k'}) \eta_k^2 
\end{align}
Choosing $\eta_k = \frac{1}{k}$, we have  
$$ \prod_{k=r}^K (1-\eta_k) \leq \exp\left(-\sum_{k=r}^K \frac{1}{k}\right) \leq \frac{r}{K} $$
We have then,

\begin{align}
        \E &\bracket{
            \Bar{F}_{q,0}(\overline{\vect{x}}_{q}) - \Bar{F}_{q,0}(\vect{x}^*)
        } \nonumber \\
        & \leq \frac{1}{K} \E \bracket{
            \Bar{F}_{q,0}(\overline{\vect{x}}_{q,1}) - \Bar{F}_{q,0}(\vect{x}^*)
        }
        + \sum_{k=1}^K \frac{k+1}{K} \cdot \frac{1}{k} \cdot \frac{1}{n} \sum_{i=1}^n \E \bracket{
            \langle \vect{\Tilde{a}}^i_{q,k}, \vect{v}^i_{q,k} - \vect{x}^*\rangle
        } \nonumber \\
        & \qquad + \sum_{k=1}^K \frac{k+1}{K}\cdot \frac{1}{k} \cdot\frac{1}{n} D \sum_{i=1}^n \E \bracket{
            \norm{
                \nabla \Bar{F}_{q,k-1} (\overline{\vect{x}}_{q,k}) - \vect{\Tilde{a}}^i_{q,k}
            }
        }
        + \frac{\beta}{2}D^2 \sum_{k=1}^K \frac{k+1}{K} \cdot \frac{1}{k^2} 
\end{align}
Which maybe simplified by using $\frac{k+1}{K} \cdot \frac{1}{k} \leq \frac{2}{K}$.

\begin{align}
        \E &\bracket{
            \Bar{F}_{q,0}(\overline{\vect{x}}_{q}) - \Bar{F}_{q,0}(\vect{x}^*)
        } \nonumber \\ 
        & \leq \frac{1}{K} \E \bracket{
            \Bar{F}_{q,0}(\overline{\vect{x}}_{q,1}) - \Bar{F}_{q,0}(\vect{x}^*)
        }
            + \frac{2}{K} \cdot \frac{1}{n}\sum_{k=1}^K \sum_{i=1}^n \E \bracket{
            \langle \widetilde{\vect{a}}^i_{q,k}, \vect{v}^i_{q,k} - \vect{x}^*\rangle
        } \nonumber \\
        & \qquad + \frac{2}{K} \cdot \frac{1}{n} D \sum_{k=1}^K \sum_{i=1}^n \E \bracket{
            \norm{
                \nabla \Bar{F}_{q,k-1} (\overline{\vect{x}}_{q,k}) - \widetilde{\vect{a}}^i_{q,k}
            }
        }
            + \frac{\beta D^2}{2} \frac{2}{K}\sum_{k=1}^K \frac{1}{k} \nonumber \\
        & \leq \frac{GD}{K}
            + \frac{2}{K} \cdot \frac{1}{n}\sum_{k=1}^K \sum_{i=1}^n \E \bracket{
            \langle \widetilde{\vect{a}}^i_{q,k}, \vect{v}^i_{q,k} - \vect{x}^*\rangle
        } \nonumber \\
        & \qquad + \frac{2}{K} \cdot D\parenthese{\beta D +  \parenthese{N + \sqrt{M}} 3K^{2/3}}
        + \frac{\beta D^2}{K}\log K 
\end{align}
where we have used \Cref{clm:bound_F_K}, \emph{G}-Lipschitz property of $\Bar{F}_{q,0}$ and boundedness of $\mathcal{K}$. Since $T = QK$ and assume that the oracle at round $k$ has a regret of order $\mathcal{O}\parenthese{\sqrt{Q}}$, i.e 
\begin{align*}
    \E \bracket{\sum_{q=1}^Q
            \langle \widetilde{\vect{a}}^i_{q,k}, \vect{v}^i_{q,k} - \vect{x}^*\rangle
        } \leq C\sqrt{Q}
\end{align*}
then, the expected regret of the algorithm upper bounded by
\begin{align}
    \E \bracket{\mathcal{R}_T} 
    &= \E \bracket{
        \sum_{q=1}^Q K \parenthese{\Bar{F}_{q,0}(\overline{\vect{x}}_{q}) - \Bar{F}_{q,0}(\vect{x}^*)}
    } \nonumber \\
    & \leq QGD + CKQ^{1/2} + 2QD\parenthese{\beta D +  \parenthese{N + \sqrt{M}} 3K^{2/3}} + Q\beta D^2 \log K \nonumber \\
    & \leq QGD + CKQ^{1/2} + 2Q\beta D^2 + 6D\parenthese{N + \sqrt{M}} QK^{2/3} +  Q\beta D^2 \log K \nonumber \\
    & \leq \parenthese{GD + 2\beta D^2}Q + CKQ^{1/2} + 6D\parenthese{N + \sqrt{M}} QK^{2/3} + Q\beta D^2\log K
\end{align}
Setting $Q = T^{2/5}$ and $K=T^{3/5}$, we have
\begin{align}
    \E \bracket{\mathcal{R}_T}
    \leq \parenthese{GD + 2\beta D^2}T^{2/5} 
        + \parenthese{C + 6D\parenthese{N + \sqrt{M}}}T^{4/5} 
        + \frac{3}{5}\beta D^2 T^{2/5}\log (T)
\end{align}
\end{proof}

\subsection{Proof of \Cref{thm:submod}}
\label{chap:submodular_analysis}

\setcounter{theorem}{12}
\begin{lemma}
\label[lemma]{lmm:submod_basic}
If $F_t$ is monotone continous DR-submodular and $\beta$-smoothness, $\vect{x}_{t,k+1} = \vect{x}_{t,k} + \frac{1}{K} \vect{v}_{t,k}$ for $k \in [1, \cdots, K]$, then
\begin{align}
    F_t(\vect{x}^*) - F_t(\vect{x}_{t,k+1}) 
    &\leq \parenthese{1 - 1/K} \bracket{F_t (\vect{x}^*) -  F_t(\vect{x}_{t,k}}) \\ \nonumber
    &- \frac{1}{K}\bracket{
        -\norm{
            \nabla F_t(\vect{x}_{t,k}) - \vect{d}_{t,k}
            }D
            + \scalarproduct{\vect{d}_{t,k}, \vect{v}_{t,k} - \vect{x}^*}
            } + \frac{\beta D^2}{2K^2}
\end{align}
\end{lemma}

\begin{proof}
The proof is essentially based on the analysis of \cite{ChenHarshaw18:Projection-Free-Online}. By $\beta$-smoothness of $F_t$,
\begin{align}
    F_{t} &(\vect{x}_{t,k+1})
    \geq F_{t}(\vect{x}_{t,k}) 
        + \scalarproduct{F_{t}(\vect{x}_{t,k}),\vect{x}_{t,k+1} - \vect{x}_{t,k}}
        - \frac{\beta}{2} \norm{\vect{x}_{t,k+1} - \vect{x}_{t,k}}^2 \nonumber \\
    & \geq F_{t}(\vect{x}_{t,k}) 
        + \frac{1}{K}  \scalarproduct{F_{t}(\vect{x}_{t,k}), \vect{v}^i_{t,k}}
        - \frac{\beta}{2} \frac{D^2}{K^2} \tag{since $\norm{\vect{v}_{t,k}} \leq D$} \\
    & \geq  F_{t}(\vect{x}_{t,k})
        + \frac{1}{K} 
            \left[\scalarproduct{\nabla F_{t}(\vect{x}_{t,k}) - \vect{d}_{t,k}, \vect{v}_{t,k} - \vect{x}^*} 
        + \scalarproduct{\nabla F_{t}(\vect{x}_{t,k}), \vect{x}^*} 
        + \scalarproduct{\vect{d}_{t,k}, \vect{v}_{t,k} - \vect{x}^*} \right]
        -  \frac{\beta}{2} \frac{D^2}{K^2} \label{eq:submod_smooth}
\end{align}
By Cauchy-Schwarz's inequality, note that,
$$ \scalarproduct{\nabla F_{t}(\vect{x}_{t,k}) - \vect{d}_{t,k}, \vect{v}_{t,k} - \vect{x}^*} \geq -\norm{\nabla F_{t}(\vect{x}_{t,k}) - \vect{t}_{t,k}} D$$
Using concavity along non-negative direction and monotonicity of $F_t$, we have,
\begin{align}
    F_t(\vect{x}^*) - F_t(\vect{x}_{t,k}) 
    &\leq F_t(\vect{x}^* \vee \vect{x}_{t,k}) - F_t(\vect{x}_{t,k}) \nonumber\\
    & \leq \scalarproduct{\nabla F_t(\vect{x}_{t,k}), (\vect{x}^* \vee \vect{x}_{t,k})      -\vect{x}_{t,k}} \nonumber\\
    & = \scalarproduct{\nabla F_t(\vect{x}_{t,k}), (\vect{x}^* - \vect{x}_{t,k}) \vee 0} \nonumber\\
    & \leq \scalarproduct{\nabla F_t(\vect{x}_{t,k}), \vect{x}^*}
\end{align}
then, \cref{eq:submod_smooth} becomes
\begin{align}
    F_{t} &(\vect{x}_{t,k+1})
    \geq F_{t}(\vect{x}_{t,k}) 
        + \scalarproduct{F_{t}(\vect{x}_{t,k}),\vect{x}_{t,k+1} - \vect{x}_{t,k}}
        - \frac{\beta}{2} \norm{\vect{x}_{t,k+1} - \vect{x}_{t,k}}^2 \nonumber \\
    &\geq  F_{t}(\vect{x}_{t,k})
        + \frac{1}{K} 
            \bracket{-\norm{\nabla F_{t}(\vect{x}_{t,k}) - \vect{t}_{t,k}} D 
            + F_t(\vect{x}^*) - F_t(\vect{x}_{t,k}) 
            + \scalarproduct{\vect{d}_{t,k}, \vect{v}_{t,k} - \vect{x}^*}}
            -  \frac{\beta}{2} \frac{D^2}{K^2} 
\end{align}
Adding and substracting $F_t(\vect{x}^*)$ and multiply both side by $-1$ yields \cref{lmm:submod_basic}.
\end{proof}

\setcounter{theorem}{3}
\begin{theorem}
Given a convex set $\mathcal{K}$ with diameters $D$. Assume that functions $F_t$ are monotone continuous DR-Submodular, $\beta$-smooth and G-Lipschitz. Setting $Q=T^{2/5}, K=T^{3/5}, T=QK$ and step-size $\eta_k = \frac{1}{K}$. Let $\rho_k = \frac{2}{\parenthese{k+3}^{2/3}}$ and $\rho_k = \frac{1.5}{\parenthese{K-k+2}^{2/3}}$ when $1 \leq k \leq \frac{K}{2}+1$ and $\frac{K}{2} +1 \leq k \leq K$ respectively. Then, the expected $\parenthese{1-\frac{1}{e}}$-regret is at most 
\begin{align}
    \E \bracket{\mathcal{R}_T} \leq \frac{3}{2}\beta D^2 T^{2/5} + \parenthese{C + 3D(N+\sqrt{M}}T^{4/5}    
\end{align}
where the constant are defined in \Cref{thm:convex}
\end{theorem}

\begin{proof}
We apply \Cref{lmm:submod_basic} with $F_t$ = $\Bar{F}_{q,k-1}$, $\vect{x}_{t,k} = \overline{\vect{x}}_{q,k}$ and $\vect{d}_{t,k} = \frac{1}{n}\sum_{i=1}^n \widetilde{\vect{a}}^i_{q,k}$, we have
\begin{align}
\label{eq:submod_upperbound}
    \Bar{F}_{q,k-1}&(\vect{x}^*) - \Bar{F}_{q,k-1}(\overline{\vect{x}}_{q,k+1}) 
    \leq \parenthese{1-\frac{1}{K}}\bracket{\Bar{F}_{q,k-1}(\vect{x}^*) - \Bar{F}_{q,k-1}(\overline{\vect{x}}_{q,k})} \nonumber\\
    & + \frac{1}{K} \cdot \frac{1}{n} \sum_{i=1}^n 
            \bracket{\norm{\nabla \Bar{F}_{q,k-1}(\overline{\vect{x}}_{q,k}) - \widetilde{\vect{a}}^i_{q,k}}D 
        + \scalarproduct{\widetilde{\vect{a}}^i_{q,k}, \vect{x}^* - \vect{v}^i_{q,k}}}
        +  \frac{\beta}{2} \frac{D^2}{K^2}
\end{align}
As $\E \bracket{\Bar{F}_{q,k-1}(\vect{x}^*) - \Bar{F}_{q,k-1}(\overline{\vect{x}}_{q,k})} = \E \bracket{\Bar{F}_{q,k-2}(\vect{x}^*) - \Bar{F}_{q,k-2}(\overline{\vect{x}}_{q,k})}$, we can apply \cref{eq:submod_upperbound} recursively for $k \in \{1, \ldots, K\}$, thus
\begin{align}
    \E &\bracket{\Bar{F}_{q,0}(\vect{x}^*) - \Bar{F}_{q,0}(\overline{\vect{x}}_{q})}
    \leq \parenthese{1-\frac{1}{K}}^K \E \bracket{\Bar{F}_{q,0}(\vect{x}^*) - \Bar{F}_{q,0}(\overline{\vect{x}}_{q,1})} \nonumber\\
    & + \frac{1}{K} \cdot \frac{1}{n} \sum_{i=1}^n \sum_{k=1}^K 
         \E \bracket{\norm{\nabla \Bar{F}_{q,k-1}(\overline{\vect{x}}_{q,k}) - \widetilde{\vect{a}}^i_{q,k}}D}
    + \frac{1}{K} \cdot \frac{1}{n} \sum_{i=1}^n \sum_{k=1}^K \E \bracket{\scalarproduct{\widetilde{\vect{a}}^i_{q,k}, \vect{x}^* - \vect{v}^i_{q,k}}}
    +  \frac{\beta}{2} \frac{D^2}{K}
\end{align}
Note that $\parenthese{1 - \dfrac{1}{K}}^K \leq \dfrac{1}{e}$ and $\Bar{F}_{q,0}(\overline{\vect{x}}_{q,1}) \geq 0$, we have 
\begin{align}
    \E&\bracket{\parenthese{1-\frac{1}{e}}\Bar{F}_{q,0}(\vect{x}^*) - \Bar{F}_{q,0}(\overline{\vect{x}}_{q})}
    \leq \frac{1}{K} \cdot \frac{1}{n} \sum_{i=1}^n \sum_{k=1}^K \E \bracket{\norm{\nabla \Bar{F}_{q,k-1}(\overline{\vect{x}}_{q,k}) - \widetilde{\vect{a}}^i_{q,k}}D} \nonumber \\
    & \quad + \frac{1}{K} \cdot \frac{1}{n} \sum_{i=1}^n \sum_{k=1}^K \E \bracket{\scalarproduct{\widetilde{\vect{a}}^i_{q,k}, \vect{x}^* - \vect{v}^i_{q,k}}}
    +  \frac{\beta}{2} \frac{D^2}{K}
\end{align}
Let $T=QK$, using \Cref{clm:bound_F_K} and note that the oracle has a regret $\mathcal{R_Q} \leq C \sqrt{Q}$. We have 
\begin{align}
    \E &\bracket{\mathcal{R}_T} 
    = \E \bracket{
        \sum_{q=1}^Q K  \bracket{\parenthese{1-\frac{1}{e}}\Bar{F}_{q,0}(\vect{x}^*) - \Bar{F}_{q,0}(\overline{\vect{x}}_{q})}
    } \nonumber \\
    & \leq \frac{D}{n} \sum_{q=1}^Q \sum_{k=1}^K \E \bracket{ \norm{\nabla \Bar{F}_{q,k-1}(\overline{\vect{x}}_{q,k}) - \widetilde{\vect{a}}^i_{q,k}}} 
        + \frac{1}{n} \sum_{q=1}^Q \sum_{i=1}^n \sum_{k=1}^K \E \bracket{\scalarproduct{\widetilde{\vect{a}}^i_{q,k}, \vect{x}^* - \vect{v}^i_{q,k}}}
        + \frac{\beta}{2} QD^2 \nonumber \\
    & \leq QD\parenthese{\beta D +  \parenthese{N + \sqrt{M}} 3K^{2/3}} 
        + KC\sqrt{Q} + \frac{\beta QD^2}{2}
\end{align}
Setting $Q = T^{2/5}$ and $K=T^{3/5}$, the expected regret of the algorithm is upper bounded by
\begin{align}
    \E \bracket{\mathcal{R}_T}
    &\leq T^{2/5}\parenthese{\beta D^2 + \parenthese{N + \sqrt{M}} 3T^{2/5}} + CT^{4/5} + \frac{\beta D^2 T^{2/5}}{2} \nonumber\\
    & \leq \frac{3}{2}\beta D^2 T^{2/5} + \parenthese{C + 3D(N+\sqrt{M}}T^{4/5}
\end{align}
\end{proof}

\section{Theoretical analysis for \Cref{chap:bandit}}
\label{chap:bandit_analysis}

Let $f^{\delta}_{t} (\vect{x})$ = $\E_{\vect{v} \in \mathbb{B}^d} \bracket{f_t \parenthese{\vect{x} + \delta \vect{v}}}$ and recall its gradient $\nabla f^{\delta}_t (\vect{x}) = \E_{\vect{u} \in \mathbb{S}^{d-1}} \bracket{\frac{d}{\delta}f_t \parenthese{\vect{x} + \delta \vect{u}}\vect{u}}$. We define the average function
\begin{align}
\label{def:average_error_Fx}
\Bar{F}_{q,k}^{\delta} (\vect{x}) = \frac{1}{L-k} \sum_{\ell=k+1}^L F_{\sigma_q (\ell)}^{\delta} (\vect{x})
= \frac{1}{L-k} \sum_{\ell=k+1}^L \frac{1}{n} \sum_{i=1}^n f^{i,\delta}_{\sigma_q (\ell)} (\vect{x})
\end{align}
and the average of the remaining \emph{$(L-k)$} functions of $f^{i,\delta}_{\sigma_q (\ell)} (\vect{x}^i_{q,\ell}$) over $n$ agents as 
\begin{align} 
\label{def:average_error_F}
\hat{F}^{\delta}_{q,k} = \frac{1}{n}\sum_{i=1}^n \hat{f}^{i,\delta}_{q,k} = \frac{1}{L-k} \sum_{\ell=k+1}^L  \frac{1}{n}\sum_{i=1}^n f^{i,\delta}_{\sigma_q (\ell)} (\vect{x}^i_{q,\ell})
\end{align}
where $F_{\sigma_q (\ell)}^{\delta} (\vect{x}) = \frac{1}{n} \sum_{i=1}^n f^{i,\delta}_{\sigma_q (\ell)} (\vect{x}) $ and $\hat{f}^{i,\delta}_{q,k} = \frac{1}{L-k} \sum_{\ell=k+1}^L f^{i,\delta}_{\sigma_q (\ell)} (\vect{x}^i_{q,\ell})$.
Then, the one-point gradient $\nabla \Bar{F}^{\delta}_{q,k}$ and $\nabla \hat{F}^{\delta}_{q,k}$ come naturally with the above definitions.
Let $\mathcal{H}_{q,1} \subset \dots \subset \mathcal{H}_{q,k}$ be the $\sigma$-fields generated by the randomness of the stochastic gradient estimate up to time $k$.
\begin{align}
\label{def:one-shot-exp}
    \vect{g}^{i, \delta}_{q,k} = \E \bracket{\widetilde{\vect{g}}_{q,k}^i \vert \mathcal{H}_{q,k-1}}, 
    \quad \vect{d}^{i, \delta}_{q,k}  = \E \bracket{\widetilde{\vect{d}}_{q,k}^i \vert \mathcal{H}_{q,k-1}}, 
    \quad \nabla f^{i,\delta}_{\sigma_q (k)} (\vect{x}^i_{q,k}) = \E \bracket{\widetilde{\vect{h}}^i_{q,k}}
\end{align}
and 
\begin{align}
\label{def:one-shot-exp-sigma}
    \hat{\vect{g}}^{i, \delta}_{q,k} = \frac{1}{L-k}\sum_{\ell=k+1}^L \vect{g}^{i, \delta}_{q,\ell},
    \quad \hat{\vect{d}}^{i, \delta}_{q,k} = \frac{1}{L-k}\sum_{\ell=k+1}^L \vect{d}^{i, \delta}_{q,\ell},
\end{align}

\setcounter{theorem}{13}
\begin{lemma}
\label[lemma]{lmm:bound_d_bandit}
For $i \in \bracket{n}, k \in \bracket{K}$. Let $V^{\delta}_d = 2n\frac{d}{\delta}B \parenthese{\frac{\lambda_2}{1-\lambda_2}+1}$, the local gradient is upper-bounded, i.e $\norm{\vect{d}^{i,\delta}_{q,k}} \leq V^{\delta}_d$
\end{lemma}

\setcounter{theorem}{14}
\begin{lemma}
\label[lemma]{lmm:stoch_variance_bandit}
Under \Cref{assum:max_bound_f}, the variance of the local gradient estimate is uniformly bounded, i.e 
\begin{align}
    \E \bracket{\norm{\vect{d}^{i,\delta}_{q,k} - \widetilde{\vect{d}}^{i, \delta}_{q,k}}^2 } \leq 4n\parenthese{\frac{d}{\delta}B}^2 \bracket{\frac{1}{\parenthese{\frac{1}{\lambda_2}-1}^2} + 2}
\end{align}
\end{lemma}
\begin{proof}
By \Cref{assum:max_bound_f}, we have 
\begin{align}
    \E &\bracket{\norm{\nabla f^{cat}_{\sigma_q(\tau)}- \widetilde{\vect{h}}^{cat}_{q,\tau} }^2 }
    = \E \bracket{\sum_{i=1}^n \norm{\nabla f^{i,\delta}_{\sigma_q(\tau)} \parenthese{\vect{x}^i_{q,\tau}} - \widetilde{\vect{h}}^i_{q,\tau}}^2}
    \leq n\parenthese{\frac{d}{\delta}B}^2
\end{align}
Following the same analysis in \cref{eq:bound_stoch_d_proof}, we have
\begin{align}
    &\E \bracket{\norm{
        \vect{d}^{cat}_{q,k}  - \widetilde{\vect{d}}^{cat}_{q,k}}^2} \nonumber \\
    \leq & \E \bracket{\parenthese{\sum_{\tau = 1}^{k-1} 
    \norm{\mathbf{W}^{k-\tau} - \frac{1}{n}\mathbf{1}_n \mathbf{1}_n^T} 
    \norm{
        \nabla f^{cat}_{\sigma_q(\tau+1)}
        - \widetilde{\vect{h}}^{cat}_{q,\tau+1} 
        + \widetilde{\vect{h}}^{cat}_{q,\tau} 
        - \nabla f^{cat}_{\sigma_q(\tau)} 
        }}^2} \nonumber \\
        & \quad 
        + 4\parenthese{\E \bracket{\norm{\mathbf{W}^{k} - \frac{1}{n}\mathbf{1}_n \mathbf{1}_n^T}^2
            \norm{
                \nabla f^{cat}_{\sigma_q(1)}
                - \widetilde{\vect{h}}^{cat}_{q,1}
            }^2} 
            + \E \bracket{\norm{\frac{1}{n}\mathbf{1}_n \mathbf{1}_n^T}^2
            \norm{
                \nabla f^{cat}_{\sigma_q(k)}
                - \widetilde{\vect{h}}^{cat}_{q,k}
            }^2}
        } \nonumber \\
    \leq & 4n\parenthese{\frac{d}{\delta}B}^2 \parenthese{\sum_{\tau = 1}^{k-1} \lambda_2^{k-\tau} }^2 
        + 4n\parenthese{\frac{d}{\delta}B}^2 \parenthese{\lambda_2^{2k} + 1} \nonumber \\
    \leq & 4n\parenthese{\frac{d}{\delta}B}^2 \parenthese{\frac{\lambda_2}{1-\lambda_2}}^2 + 4n\parenthese{\frac{d}{\delta}B}^2(\lambda_2 + 1) 
    \leq 4n\parenthese{\frac{d}{\delta}B}^2 \bracket{\frac{1}{\parenthese{\frac{1}{\lambda_2}-1}^2} + 2}
\end{align}
The lemma follows by remarking that $\E \bracket{\norm{\vect{d}^{i,\delta}_{q,k} - \widetilde{\vect{d}}^{i}_{q,k}}^2} \leq \E \bracket{\norm{
        \vect{d}^{cat}_{q,k}  - \widetilde{\vect{d}}^{cat}_{q,k}}^2}$
\end{proof}

Let $\vect{x}^* = \argmax_{\vect{x} \in \mathcal{K}} \sum_{t=1}^T f_{t} (\vect{x}) $, $\vect{x}^*_{\delta} = \argmax_{\vect{x} \in \mathcal{K}'} \sum_{t=1}^T f_t (\vect{x})$
Let $\vect{z}^i_{q,k} = \vect{x}^i_{q,k} + \delta \vect{u}^i_{q,k}$, we define $\overline{\vect{z}}_{q,k} = \frac{1}{n}\sum_{i=1}^n \vect{z}^i_{q,k}$ for $1 \leq k \leq K$.  


\begin{lemma}
\label[lemma]{lmm:bound_local_global_d_bandit}
Under \Cref{assum:assum_1} and \Cref{assum:max_bound_f}. Let $N = k_0\cdot n B \frac{d}{\delta}\max \curlybracket{\lambda_2\parenthese{1 + \frac{2}{1-\lambda_2}}, 2}$. Then, for $k \in \bracket{K}$,we have
\begin{align}
    \max_{i \in \bracket{1,n}} \E \bracket{ \norm{
        \hat{\vect{d}}^{i,\delta}_{q,k} - \nabla \hat{F}^{\delta}_{q,k}}
    } \leq \frac{N}{k}
\end{align}
\end{lemma}

\begin{proof}
The proof is essentially based on the one of \Cref{lmm:bound_d_avg_consensus}. Note that we keep the same notation with a superscipt $\delta$ to indicate the smooth version of $f$ and related variables. By definition of the one-point gradient estimator and \Cref{assum:max_bound_f}, \cref{eq:diff_k} becomes
\begin{align}
       \E \bracket{\norm{\delta^{cat, \delta}_{q,k}}^2}
    &= \E \bracket{\sum_{i=1}^n \norm{\delta ^{i,\delta}_{q,k}}^2} 
    = \sum_{i=1}^n \E  
    \bracket{
        \norm{
            \nabla \hat{f}^{i,\delta}_{q,k} - \nabla \hat{f}^{i,\delta}_{q,k-1}
        }^2
    } \nonumber\\
    & \quad = \sum_{i=1}^n \E 
        \bracket{
            \E 
                \bracket{ 
                    \norm{
                        \nabla \hat{f}^{i,\delta}_{q,k} - \nabla \hat{f}^{i,\delta}_{q,k-1}
                    }^2
                    \bigm\vert \mathcal{F}_{q,k-1}
                }
        } \nonumber\\
    & \quad = \sum_{i=1}^n \E \bracket{
            \E 
                \bracket{ 
                    \norm{
                        \frac{\sum_{\ell=k+1}^L \nabla f^{i,\delta}_{\sigma_q (\ell)} (\vect{x}^i_{q,\ell})}{L-k} 
                        - \frac{\sum_{\ell=k}^L \nabla f^{i,\delta}_{\sigma_q (\ell)} (\vect{x}^i_{q,\ell})}{L-k+1}
                        }^2
                        \bigm\vert  \mathcal{F}_{q,k-1}
                }
            } \nonumber \\
    & \quad = \sum_{i=1}^n\E \bracket{
                \E \bracket{ \norm{
                    \frac{\sum_{\ell=k+1}^L \nabla f^{i,\delta}_{\sigma_q (\ell)} (\vect{x}^{i}_{q,\ell})}{(L-k)(L-k+1)}
                     - \frac{\nabla f^{i,\delta}_{\sigma_q (k)} (\vect{x}^i_{q,k})}{{L-k+1}}}^2
                     \bigm\vert \mathcal{F}_{q,k-1} }
            }
         \nonumber\\
    &  \quad \leq 
            n \left( 
                \frac{2B\frac{d}{\delta}}{L-k+1}
            \right)^2 \label{eq:diff_k_bandit}
\end{align}
By Jensen's inequality, we deduce that 
\begin{align}
\label{eq:bound_delta_cat_bandit}
    \E \bracket{\norm{\delta^{cat,\delta}_{q,k}}}
    \leq \sqrt{\E \bracket{\norm{\delta^{cat,\delta}_{q,k}}^2}} 
    \leq \frac{2\sqrt{n}B\frac{d}{\delta}}{L-k+1}
\end{align}
When $k=1$, following the same derivation in \cref{eq:bound_cat_k1}, we have
\begin{align}
    \E& \bracket{\norm{\vect{\hat{d}}^{cat,\delta}_{q,1} - \nabla \hat{F}^{cat,\delta}_{q,1}}^2}
    \leq \lambda_2^2 \E \bracket{\sum_{i=1}^n \norm{
        \vect{\hat{g}}^{i, \delta}_{q,1} - \nabla \hat{F}^{\delta}_{q,1}}^2
    } \nonumber 
    \leq n\lambda_2^2 \frac{d^2}{\delta^2}B^2
\end{align}
Let $k \in \bracket{2,k_0}$, from \cref{eq:recurrence_d_F} and \cref{eq:bound_delta_cat_bandit}
\begin{align}
    \E \bracket{\norm{\vect{\hat{d}}^{cat, \delta}_{q,k} - \nabla \hat{F}^{cat, \delta}_{q,k}}}
     & \leq \lambda_2 \parenthese{
         \E \bracket{\norm{\vect{\hat{d}}^{cat,\delta}_{q,k-1} - \nabla \hat{F}^{cat,\delta}_{q,k-1}}}
        + \E \bracket{\norm{\delta^{cat,\delta}_{q,k}}}  
        } \nonumber \\
    & \leq \lambda_2^{k-1}\sqrt{n}\frac{d}{\delta}B + 2 \sum_{\tau=1}^{k} \lambda_2^{\tau} \sqrt{n}\frac{d}{\delta}B \nonumber\\
    & \leq \lambda_2\sqrt{n}\frac{d}{\delta}B + 2\frac{\lambda_2}{1-\lambda_2}\sqrt{n}\frac{d}{\delta}B \nonumber \\
    &= \lambda_2 \sqrt{n}\frac{d}{\delta}B \parenthese{1 + \frac{2}{1-\lambda_2}}
\end{align}
Let $N_0 = k_0\cdot\sqrt{n}\max \curlybracket{\lambda_2 B\frac{d}{\delta} \parenthese{1 + \frac{2}{1-\lambda_2}}, 2B\frac{d}{\delta}}$. We claim that $\E \bracket{\norm{\vect{\hat{d}}^{cat, \delta}_{q,k} - \nabla \hat{F}^{cat, \delta}_{q,k}}} \leq \frac{N_0}{k}$ when $k \in \bracket{k_0, K}$. Let $L \geq 2K$, we have then $\frac{1}{L-k+1} \leq \frac{1}{2K-k+1} \leq \frac{1}{K+1} \leq \frac{1}{k+1}$. Thus, using the induction hypothesis, we have
\begin{align}
    \E \bracket{\norm{\vect{\hat{d}}^{cat, \delta}_{q,k} - \nabla \hat{F}^{cat, \delta}_{q,k}}}
     & \leq \lambda_2 \parenthese{
         \E \bracket{\norm{\vect{\hat{d}}^{cat,\delta}_{q,k-1} - \nabla \hat{F}^{cat,\delta}_{q,k-1}}}
        + \E \bracket{\norm{\delta^{cat,\delta}_{q,k}}}  
        } \nonumber \\
    & \leq \lambda_2 \parenthese{\frac{N_0}{k-1} + \frac{2\sqrt{n}B\frac{d}{\delta}}{L-k+1}} \nonumber\\
    & \leq \lambda_2 \parenthese{\frac{N_0}{k-1} + \frac{2\sqrt{n}B\frac{d}{\delta}}{k+1}} \nonumber \\
    & \leq \lambda_2 \left( 
        N_0\frac{k_0 + 1}{k_0 (k-1)}
    \right) \nonumber \\
    & \leq \frac{N_0}{k} \label{eq:boundksmall_bandit}
\end{align}
Using the inequality in \cref{eq:bound_sum_sqrt} and the above result, the lemma is then proven.

\end{proof}

\begin{lemma}[Lemma 10, Lemma 11 \cite{Zhang:2019}]
\label[lemma]{lmm:lmm:bound_d_var_red_bandit}
Under \Cref{lmm:bound_d_bandit} and \cref{lmm:stoch_variance_bandit} and setting $\rho_k = \frac{2}{\parenthese{k+3}^{2/3}}$, we have
\begin{align}
    \E \bracket{
        \norm{\hat{\vect{d}}^{i,\delta}_{q,k-1} - \widetilde{\vect{a}}^i_{q,k}}} 
    \leq \frac{\sqrt{M_0}}{\parenthese{k+3}^{1/3}}, \qquad k \in \bracket{K}
\end{align}
where $M_0 = 4^{2/3}\frac{d^2}{\delta^2}B^2 \bracket{24n^2 \parenthese{\frac{1}{\frac{1}{\lambda_2}-1} + 1}^2 + 8n\parenthese{\frac{1}{\parenthese{\frac{1}{\lambda_2}-1}^2}+2}}$ 
\end{lemma}

\begin{proof}
The proof follows the same idea in Lemma 10 and Lemma 11 of \cite{Zhang:2019} with some changes in the constant values. We will evoques in details in the following section. Following the same decomposition in the proof of \Cref{lmm:bound_d_var_red}, we have
\begin{align*}
    \E &\bracket{\norm{\hat{\vect{d}}^{i,\delta}_{q,k-1} - \widetilde{\vect{a}}^i_{q,k}}^2}
    = \E \bracket{\norm{
        \hat{\vect{d}}^{i,\delta}_{q,k-1} 
        - (1-\rho_k) \widetilde{\vect{a}}^i_{q,k-1} 
        - \rho_k \widetilde{\vect{d}}^{i,\delta}_{q,k}
    }^2
    }\\
    & = \rho_k^2 \E \bracket{\norm{\hat{\vect{d}}^{i,\delta}_{q,k-1}-\widetilde{\vect{d}}^{i,\delta}_{q,k})}^2}
    + (1-\rho_k)^2 \E \bracket{\norm{\hat{\vect{d}}^{i,\delta}_{q,k-1} - \hat{\vect{d}}^{i,\delta}_{q,k-2}}^2} \\
    & \quad + (1-\rho_k)^2 
        \E \bracket{\norm{\hat{\vect{d}}^{i,\delta}_{q,k-2} - \widetilde{\vect{a}}^i_{q,k-1} }^2} \\
    & \quad + 2\rho_k (1-\rho_k) 
        \E \bracket{\scalarproduct{ 
            \hat{\vect{d}}^{i,\delta}_{q,k-1} - \widetilde{\vect{d}}^{i,\delta}_{q,k}, \hat{\vect{d}}^{i,\delta}_{q,k-1} - \hat{\vect{d}}^{i,\delta}_{q,k-2} 
        }} \\
    & \quad + 2\rho_k (1-\rho_k) 
        \E \bracket{\scalarproduct{ 
             \hat{\vect{d}}^{i,\delta}_{q,k-1} - \widetilde{\vect{d}}^{i,\delta}_{q,k},
             \hat{\vect{d}}^{i,\delta}_{q,k-2} - \widetilde{\vect{a}}^i_{q,k-1}
        }} \\
    & \quad + 2(1-\rho_k)^2 
        \E \bracket{\scalarproduct{
            \hat{\vect{d}}^{i,\delta}_{q,k-1} - \hat{\vect{d}}^{i,\delta}_{q,k-2},
            \hat{\vect{d}}^{i,\delta}_{q,k-2} - \widetilde{\vect{a}}^i_{q,k-1}
        }}
\end{align*}
\begin{align}
    \E &\bracket{
        \norm{
            \hat{\vect{d}}^{i,\delta}_{q,k-1}-\widetilde{\vect{d}}^{i,\delta}_{q,k})
        }^2
    }
    = \E \bracket{\E \bracket{
        \norm{
            \hat{\vect{d}}^{i,\delta}_{q,k-1}-\widetilde{\vect{d}}^{i,\delta}_{q,k})
        }^2 \bigm\vert \mathcal{F}_{q,k-1}
    }} \nonumber\\
    & \leq \E \bracket{\E \bracket{
        \norm{
            \hat{\vect{d}}^{i,\delta}_{q,k-1}
            - \vect{d}^{i,\delta}_{q,k}}^2
        + \norm{ 
            \vect{d}^{i,\delta}_{q,k}
            -\widetilde{\vect{d}}^{i,\delta}_{q,k})}^2
        + 2 \langle
                \hat{\vect{d}}^{i,\delta}_{q,k-1} - \vect{d}^{i,\delta}_{q,k},
                \vect{d}^{i,\delta}_{q,k} -\widetilde{\vect{d}}^{i,\delta}_{q,k}
            \rangle
    \bigm\vert \mathcal{F}_{q,k-1} 
    }} \label{eq:bound_d0_bandit}
\end{align}
By the definition in \cref{def:one-shot-exp-sigma}, we have $\E\bracket{\vect{d}^{i,\delta}_{q,k} \vert \mathcal{F}_{q,k-1}} = \hat{\vect{d}}^{i,\delta}_{q,k-1}$, using \Cref{lmm:bound_d_bandit}, we have
\begin{align}
    \E \bracket{\E \bracket{
        \norm{
            \hat{\vect{d}}^{i,\delta}_{q,k-1}
            - \vect{d}^{i,\delta}_{q,k}}^2
        \bigm\vert \mathcal{F}_{q,k-1} 
    }} 
    & \leq (V^{\delta}_{\vect{d}})^2 
\end{align}
Invoking \Cref{lmm:stoch_variance_bandit}, we have
\begin{align}
    \E \bracket{ \norm{
        \vect{d}^{i,\delta}_{q,k} -\widetilde{\vect{d}}^{i,\delta}_{q,k})
    }^2 
    }
    \leq \sigma^2_2 \label{eq:bound_d2_bandit}
\end{align}
and 
\begin{align}
    &\E \bracket{\E \bracket{
        \scalarproduct{
                \hat{\vect{d}}^{i,\delta}_{q,k-1} - \vect{d}^{i,\delta}_{q,k},
                \vect{d}^{i,\delta}_{q,k} -\widetilde{\vect{d}}^{i,\delta}_{q,k}
        }
        \bigm\vert \mathcal{F}_{q,k-1}
    }} = 0 \nonumber \\
\end{align}
by following the same analysis in \cref{eq:bound_d3}. We now claim that \cref{eq:bound_d0_bandit} is bounded above by
\begin{align}
    \label{eq:bound_d_hat_tilde_bandit}
    \E &\bracket{
        \norm{
            \hat{\vect{d}}^{i,\delta}_{q,k-1}-\widetilde{\vect{d}}^{i,\delta}_{q,k})
        }^2
    } \leq (V^{\delta}_{\vect{d}})^2 + \sigma^2_2 \triangleq V^{\delta}
\end{align}
More over, taking the idea from \cref{eq:bound_d_hat_diff,eq:bound_d4,eq:bound_d5,eq:bound_d6}, we have
\begin{align}
    \label{eq:bound_d_hat_diff_bandit}
    \E \bracket{\norm{
        \hat{\vect{d}}^{i,\delta}_{q,k-1} - \hat{\vect{d}}^{i,\delta}_{q,k-2} 
    }^2}
    \leq \frac{4(V^{\delta}_{\vect{d}})^2}{\left(L-k+2\right)^2}
    \triangleq \frac{L^{\delta}}{\left(L-k+2\right)^2}
\end{align}

\begin{align}
\label{eq:bound_d4_bandit}
    &\E \bracket{\scalarproduct{
            \hat{\vect{d}}^{i,\delta}_{q,k-1} - \widetilde{\vect{d}}^{i,\delta}_{q,k}, \hat{\vect{d}}^{i,\delta}_{q,k-1} - \hat{\vect{d}}^{i,\delta}_{q,k-2} 
        }} = 0
\end{align}

\begin{align}
\label{eq:bound_d5_bandit}
    &\E \bracket{\scalarproduct{ 
             \hat{\vect{d}}^{i,\delta}_{q,k-1} - \widetilde{\vect{d}}^{i,\delta}_{q,k},
             \hat{\vect{d}}^{i,\delta}_{q,k-2} - \widetilde{\vect{a}}^i_{q,k-1}
        }} 
    = 0
\end{align}
and 
\begin{align}
    \label{eq:bound_d6_bandit}
    \E &\bracket{\scalarproduct{
            \hat{\vect{d}}^{i,\delta}_{q,k-1} - \hat{\vect{d}}^{i,\delta}_{q,k-2},
            \hat{\vect{d}}^{i,\delta}_{q,k-2} - \widetilde{\vect{a}}^i_{q,k-1}
        }} \leq \frac{L^{\delta}}{2\alpha_k (L-k+2)^2} 
        + \frac{\alpha_k}{2} \E \bracket{\norm{\hat{\vect{d}}^{i,\delta}_{q,k-2} - \widetilde{\vect{a}}^i_{q,k-1}}^2}
\end{align}
by using Young's inequality. Setting $\alpha_k = \frac{\rho_k}{2}$ similarly to \Cref{lmm:bound_d_var_red}, we have 
\begin{align}
    \label{eq:recurrent_psi_bandit}
    \E \bracket{\norm{\hat{\vect{d}}^{i,\delta}_{q,k-1} - \widetilde{\vect{a}}^i_{q,k}}^2} 
    & \leq \rho_k^2 V^{\delta} 
        + \parenthese{1+\frac{2}{\rho_k}} \frac{L^{\delta}}{\parenthese{L-k+2}^2} 
        + \parenthese{1-\rho_k} \E \bracket{\norm{\hat{\vect{d}}^{i,\delta}_{q,k-2} - \widetilde{\vect{a}}^i_{q,k-1}}^2}
\end{align}
Setting $L \geq 2K$ and $\rho_k = \frac{2}{(k+2)^{2/3}}$, we have then $\frac{1}{L-k+2} \leq \frac{1}{2K-k+2} \leq \frac{1}{K+2} \leq \frac{1}{k+2}$. Following the derivation from Lemma 11 of \cite{Zhang:2019}.\Cref{eq:recurrent_psi_bandit} can be bounded above by
\begin{align}
    \label{eq:recurrent_psi_bandit2}
    \E \bracket{\norm{\hat{\vect{d}}^{i,\delta}_{q,k-1} - \widetilde{\vect{a}}^i_{q,k}}^2} 
    & \leq \rho_k^2 V^{\delta} 
        + \parenthese{1+\frac{2}{\rho_k}} \frac{L^{\delta}}{\parenthese{k+2}^2} 
        + \parenthese{1-\rho_k} 
        \E \bracket{\norm{\hat{\vect{d}}^{i,\delta}_{q,k-2} - \widetilde{\vect{a}}^i_{q,k-1}}^2} \nonumber \\
    & \leq \frac{4^{2/3}\parenthese{2V^{\delta} + L^{\delta}}}{(k+2)^{4/3}}
        + \parenthese{1-\frac{2}{(k+2)^{2/3}}}
        \E \bracket{\norm{\hat{\vect{d}}^{i,\delta}_{q,k-2} - \widetilde{\vect{a}}^i_{q,k-1}}^2} \nonumber \\
    & \triangleq \frac{M_0}{(k+2)^{4/3}} + \parenthese{1-\frac{2}{(k+2)^{2/3}}}\E \bracket{\norm{\hat{\vect{d}}^{i,\delta}_{q,k-2} - \widetilde{\vect{a}}^i_{q,k-1}}^2}
\end{align}
Assume that $\E \bracket{\norm{\hat{\vect{d}}^{i,\delta}_{q,k-1} - \widetilde{\vect{a}}^i_{q,k}}^2} \leq \frac{M_0}{(k+3)^{2/3}}$ for $k \in \bracket{K}$. When $k=1$, by definition of $\widetilde{\vect{a}}^{i}_{q,1}$ and $\hat{\vect{d}}^{i,\delta}_{q,0}$, we have
\begin{align}
    \E \bracket{\norm{\hat{\vect{d}}^{i,\delta}_{q,0} - \widetilde{\vect{a}}^i_{q,1}}^2} 
    \leq \parenthese{
        V^{\delta}_{\vect{d}} + \frac{2}{3^{2/3}} \frac{d}{\delta}B
    }^2
\end{align}
Thus, since $\sigma_2 \geq \frac{2}{3^{2/3}} \frac{d}{\delta}B$, one can observe that  
\begin{align}
    \frac{M_0}{(1+2)^{2/3}} = 2V^{\delta} + L^{\delta} \geq 2V^{\delta} = 2\parenthese{(V^{\delta}_{\vect{d}})^2 + \sigma^2_2} \geq \parenthese{V^{\delta}_{\vect{d}}) + \sigma_2}^2 \geq \E \bracket{\norm{\hat{\vect{d}}^{i,\delta}_{q,0} - \widetilde{\vect{a}}^i_{q,1}}^2}
\end{align}
Suppose that the induction hypothesis holds for $k-1$, one can easily verify for $k$ since
\begin{align}
    \E \bracket{\norm{\hat{\vect{d}}^{i,\delta}_{q,k-1} - \widetilde{\vect{a}}^i_{q,k}}^2} 
    & \leq \frac{M_0}{(k+2)^{4/3}} + \parenthese{1-\frac{2}{(k+2)^{2/3}}}\E \bracket{\norm{\hat{\vect{d}}^{i,\delta}_{q,k-2} - \widetilde{\vect{a}}^i_{q,k-1}}^2} \nonumber \\
    & \leq \frac{M_0}{(k+2)^{4/3}} + \parenthese{1-\frac{2}{(k+2)^{2/3}}}
    \frac{M_0}{(k+2)^{2/3}} \nonumber \\
    & \leq M_0\frac{(k+2)^{2/3} - 1}{(k+3)^{4/3}} \nonumber \\
    & \leq \frac{M_0}{(k+3)^{2/3}}
\end{align}
\end{proof}

\begin{claim}
\label{clm:bandit_regret1}
Under \Cref{clm:f_hat_f_bar}, \Cref{lmm:bound_local_global_d_bandit} and \Cref{lmm:lmm:bound_d_var_red_bandit}.
\begin{align}
    \sum_{k=1}^K \E \bracket{\norm{\nabla \Bar{F}^{\delta}_{q,k-1}(\overline{\vect{x}}_{q,k}) - \widetilde{\vect{a}}^i_{q,k}}} \leq \beta D + \frac{3}{2}\parenthese{N + \sqrt{M_0}}K^{2/3}
\end{align}
\end{claim}

\begin{claimproof}
\begin{align}
    \label{eq:bandit_bound_with_K}
    \sum_{k=1}^K  \E \bracket{\norm{\nabla \Bar{F}^{\delta}_{q,k-1}(\overline{\vect{x}}_{q,k}) - \vect{\Tilde{a}}_{q,k}^i}} \nonumber 
    & \leq \sum_{k=1}^K \E \bracket{ \norm{\nabla \Bar{F}^{\delta}_{q,k-1}(\overline{\vect{x}}_{q,k}) -  \nabla \hat{F}^{\delta}_{q,k-1}}}
        + \sum_{k=1}^K \E \bracket{\norm{\nabla \hat{F}^{\delta}_{q,k-1} - \vect{\widetilde{a}}_{q,k}^i }} \nonumber \\
    & \leq \beta D
        + \sum_{k=1}^K \E \bracket{\norm{\nabla \hat{F}^{\delta}_{q,k-1} - \hat{\vect{d}}^{i,\delta}_{q,k-1}}}
        + \sum_{k=1}^K \E \bracket{\norm{ \hat{\vect{d}}^{i,\delta}_{q,k-1} - \vect{\widetilde{a}}_{q,k}^i}} \nonumber \\
    & \leq \beta D
        + \sum_{k=1}^K \frac{N}{k}
        + \sum_{k=1}^K \frac{\sqrt{M_0}}{\parenthese{k+3}^{1/3}} \nonumber \\
    & \leq \beta D + \parenthese{N + \sqrt{M_0}}\sum_{k=1}^K \frac{1}{(k+3)^{1/3}}  \nonumber \\
    & \leq \beta D + \frac{3}{2}\parenthese{N + \sqrt{M_0}}K^{2/3}
\end{align}
where \Cref{clm:f_hat_f_bar} is still verified in the second inequality since $f^{i,\delta}_{\sigma_q(\ell)}$ is $\beta$-smooth and the third inequality is the result of \Cref{lmm:bound_local_global_d_bandit} and \Cref{lmm:lmm:bound_d_var_red_bandit}
\end{claimproof}


\begin{claim}
\label{clm:oneshot_regret}
\begin{align}
    \E \bracket{\sum_{q=1}^Q \sum_{\ell=1}^L \parenthese{1-\frac{1}{e}}F^{\delta}_{\sigma_q(\ell)}(\vect{x}^*_{\delta}) - F^{\delta}_{\sigma_q(\ell)}(\overline{\vect{x}}_{q})}
    \leq \frac{L \beta D^2}{K} + \frac{3LD\parenthese{N + \sqrt{M_0}}}{2K^{1/3}} + LC\sqrt{Q} + \frac{\beta QLD^2}{2K}
\end{align}
\end{claim}
\begin{claimproof}
Using \Cref{lmm:submod_basic} with $F_t$ = $\Bar{F}^{\delta}_{q,k-1}$, $\vect{x}_{t,k} = \overline{\vect{x}}_{q,k}$ and $\vect{d}_{t,k} = \frac{1}{n}\sum_{i=1}^n \widetilde{\vect{a}}^i_{q,k}$, we have
\begin{align}
\label{eq:submod_upperbound_bandit}
    \Bar{F}^{\delta}_{q,k-1}&(\vect{x}^*_{\delta}) - \Bar{F}^{\delta}_{q,k-1}(\overline{\vect{x}}_{q,k+1}) 
    \leq \parenthese{1-\frac{1}{K}}\bracket{\Bar{F}^{\delta}_{q,k-1}(\vect{x}^*_{\delta}) - \Bar{F}^{\delta}_{q,k-1}(\overline{\vect{x}}_{q,k})} \nonumber\\
    & + \frac{1}{K} \cdot \frac{1}{n} \sum_{i=1}^n 
            \left[\norm{\nabla \Bar{F}^{\delta}_{q,k-1}(\overline{\vect{x}}_{q,k}) - \widetilde{\vect{a}}^i_{q,k}}D 
        + \scalarproduct{\widetilde{\vect{a}}^i_{q,k}, \vect{x}^*_{\delta} - \vect{v}^i_{q,k}} \right]
        +  \frac{\beta}{2} \frac{D^2}{K^2}
\end{align}
Similarly to the proof of \Cref{thm:submod} and using \Cref{clm:bandit_regret1}, we note
\begin{align}
    &\E \bracket{\parenthese{1-\frac{1}{e}}\Bar{F}^{\delta}_{q,0}(\vect{x}^*_{\delta}) - \Bar{F}^{\delta}_{q,0}(\overline{\vect{x}}_{q})} \nonumber \\
    & \leq \frac{1}{K} \cdot \frac{1}{n} \sum_{i=1}^n \sum_{k=1}^K \E \bracket{\norm{\nabla \Bar{F}^{\delta}_{q,k-1}(\overline{\vect{x}}_{q,k}) - \widetilde{\vect{a}}^i_{q,k}}D} 
        + \frac{1}{K} \cdot \frac{1}{n} \sum_{i=1}^n \sum_{k=1}^K \E \bracket{\scalarproduct{\widetilde{\vect{a}}^i_{q,k}, \vect{x}^*_{\delta} - \vect{v}^i_{q,k}}}
        + \frac{\beta}{2} \frac{D^2}{K} \nonumber \\
    & \leq \frac{D}{K} \parenthese{\beta D + \frac{3}{2}\parenthese{N + \sqrt{M_0}}K^{2/3}} 
        + \frac{1}{K} \cdot \frac{1}{n} \sum_{i=1}^n \sum_{k=1}^K \E \bracket{\scalarproduct{\widetilde{\vect{a}}^i_{q,k}, \vect{x}^*_{\delta} - \vect{v}^i_{q,k}}}
        + \frac{\beta}{2} \frac{D^2}{K} \nonumber \\
\end{align}
Thus, we can write 
\begin{align}
    &\E \bracket{\sum_{q=1}^Q \sum_{\ell=1}^L \parenthese{1-\frac{1}{e}}F^{\delta}_{\sigma_q(\ell)}(\vect{x}^*_{\delta}) - F^{\delta}_{\sigma_q(\ell)}(\overline{\vect{x}}_{q})} 
    = \E \bracket{\sum_{q=1}^Q L \parenthese{1-\frac{1}{e}}\Bar{F}^{\delta}_{q,0}(\vect{x}^*_{\delta}) - \Bar{F}^{\delta}_{q,0}(\overline{\vect{x}}_{q})} \nonumber \\
    & \leq \frac{LD}{K} \parenthese{\beta D + \frac{3}{2}\parenthese{N + \sqrt{M_0}}K^{2/3}} 
        + LC\sqrt{Q}
        + \frac{\beta}{2} \frac{QLD^2}{K} \nonumber \\
    & \leq \frac{L \beta D^2}{K} + \frac{3LD\parenthese{N + \sqrt{M_0}}}{2K^{1/3}} + LC\sqrt{Q} + \frac{\beta QLD^2}{2K}
\end{align}
\end{claimproof}

\setcounter{theorem}{6}
\begin{theorem}
Let $\mathcal{K}$ be a down-closed convex and compact set. We suppose the $\delta$-interior $\mathcal{K}'$ following $\Cref{lmm:discrepancy}$. Let $Q = T^{2/9}, L = T^{7/9}, K = T^{2/3} $, $\delta = \frac{r}{\sqrt{d}+2}T^{-1/9}$ and $\rho_k = \frac{2}{(k+2)^{2/3}}$, $\eta_k = \frac{1}{K}$. Then the expected $\parenthese{1-\frac{1}{e}}$-regret is upper bounded 
\begin{align}
    \E \bracket{\mathcal{R}_T} \leq ZT^{8/9} + \frac{\beta D^2}{2}T^{1/9} 
                + \frac{3}{2} D \frac{d \parenthese{\sqrt{d}+2}}{r} P_{n,\lambda_2} T^{2/9} + \beta D^2 T^{3/9}
\end{align}
where we note
$    Z = \parenthese{1-\frac{1}{e}} \parenthese{\sqrt{d}\parenthese{\frac{R}{e} +1} + \frac{R}{r}} G \frac{r}{\sqrt{d}+2} 
            + \parenthese{2-\frac{1}{e}}G \frac{r}{\sqrt{d}+2}+ 2\beta + C$
and 
$
    P_{n,\lambda_2} = k_0 \cdot n B\max \curlybracket{\lambda_2\parenthese{1 + \frac{2}{1-\lambda_2}}, 2} 
    + 4^{1/3}\parenthese{24n^2 \parenthese{\frac{1}{\frac{1}{\lambda_2}-1} + 1}^2 + 8n\parenthese{\frac{1}{\parenthese{\frac{1}{\lambda_2}-1}^2}+2}}^{1/2}
$
\end{theorem}

\begin{proof}
Recall the values of $N$ and $M_0$ from \Cref{lmm:bound_local_global_d_bandit} and \Cref{lmm:lmm:bound_d_var_red_bandit}, we have
\begin{align*}N = k_0\cdot n B \frac{d}{\delta}\max \curlybracket{\lambda_2\parenthese{1 + \frac{2}{1-\lambda_2}}, 2}
\end{align*}
\begin{align*}M_0 = 4^{2/3}\frac{d^2}{\delta^2}B^2 \bracket{24n^2 \parenthese{\frac{1}{\frac{1}{\lambda_2}-1} + 1}^2 + 8n\parenthese{\frac{1}{\parenthese{\frac{1}{\lambda_2}-1}^2}+2}}
\end{align*}
Let $P_{n,\lambda_2} = k_0\cdot nB\max \curlybracket{\lambda_2\parenthese{1 + \frac{2}{1-\lambda_2}}, 2} 
    + 4^{1/3}\parenthese{24n^2 \parenthese{\frac{1}{\frac{1}{\lambda_2}-1} + 1}^2 + 8n\parenthese{\frac{1}{\parenthese{\frac{1}{\lambda_2}-1}^2}+2}}^{1/2}$. Then, 
one can easily see that $N + \sqrt{M_0} = \frac{d}{\delta}B P_{n, \lambda_2}$
where $P_{n,\lambda_2}$ is a constant depending on $n$ and $\lambda_2$. For the next step, we set $\delta = \frac{r}{\sqrt{d}+2}T^{-1/9}$, then $\frac{d}{\delta} = \frac{d\parenthese{\sqrt{d}+2}}{r} T^{1/9}$, $Q = T^{2/9}, L = T^{7/9}$ and $K = T^{2/3}$. From the analysis in Theorem 4 of \cite{Zhang:2019}, \Cref{lmm:discrepancy} and \Cref{clm:oneshot_regret}, we have
\begin{align}
    \E \bracket{\mathcal{R_T}} 
    &\leq \parenthese{1-\frac{1}{e}}\parenthese{\sqrt{d}\parenthese{\frac{R}{e} +1} + \frac{R}{r}} GT \delta^{\gamma} 
        + \parenthese{2-\frac{1}{e}}GT\delta
        + 2BQK \nonumber \\
        & \quad + \sum_{q=1}^Q \sum_{\ell=1}^L \parenthese{1-\frac{1}{e}}F^{\delta}_{\sigma_q(\ell)}(\vect{x}^*_{\delta}) - F^{\delta}_{\sigma_q(\ell)}(\overline{\vect{x}}_{q}) \nonumber \\
    &\leq \parenthese{1-\frac{1}{e}} \parenthese{\sqrt{d}\parenthese{\frac{R}{e} +1} + \frac{R}{r}} GT \delta^{\gamma} 
        + \parenthese{2-\frac{1}{e}}GT\delta
        + 2BQK \nonumber \\
        & \quad + \frac{L \beta D^2}{K} + \frac{3LD\parenthese{N + \sqrt{M_0}}}{2K^{1/3}} + LC\sqrt{Q} + \frac{\beta QLD^2}{2K} \nonumber \\
    &\leq \parenthese{1-\frac{1}{e}} \parenthese{\sqrt{d}\parenthese{\frac{R}{e} +1} + \frac{R}{r}} GT \delta^{\gamma} 
        + \parenthese{2-\frac{1}{e}}GT\delta
        + 2BQK \nonumber \\
        & \quad + \frac{L \beta D^2}{K} + \frac{3LD\frac{d}{\delta}P_{n,\lambda_2}}{2K^{1/3}} + LC\sqrt{Q} + \frac{\beta QLD^2}{2K} \nonumber \\
    &\leq \parenthese{1-\frac{1}{e}} \parenthese{\sqrt{d}\parenthese{\frac{R}{e} +1} + \frac{R}{r}} G T \frac{r}{\sqrt{d}+2} T^{-1/9}
    + \parenthese{2-\frac{1}{e}}G T \frac{r}{\sqrt{d}+2}T^{-1/9} \nonumber \\
        &\quad  + 2\beta T^{2/9}T^{2/3} 
            + T^{7/9}\beta D^2 T^{-2/3}
            + \frac{3}{2} T^{7/9} D \frac{d\parenthese{\sqrt{d}+2}}{r} T^{1/9} P_{n, \lambda_2} T^{-2/3} \nonumber \\
        &\quad  + T^{7/9}C T^{1/9} 
                + \frac{\beta}{2} T^{2/9} T^{7/9} D^2 T^{-2/3} \nonumber \\
    &\leq \bracket{
            \parenthese{1-\frac{1}{e}} \parenthese{\sqrt{d}\parenthese{\frac{R}{e} +1} + \frac{R}{r}} G \frac{r}{\sqrt{d}+2} 
            + \parenthese{2-\frac{1}{e}}G \frac{r}{\sqrt{d}+2} + C} T^{8/9} \nonumber \\
        &\quad + \frac{\beta D^2}{2}T^{6/9} 
                + \bracket{2\beta + \frac{3}{2} D \frac{d \parenthese{\sqrt{d}+2}}{r} P_{n,\lambda_2}} T^{5/9} + \beta D^2 T^{4/9} \nonumber \\
    &\leq \bracket{
            \parenthese{1-\frac{1}{e}} \parenthese{\sqrt{d}\parenthese{\frac{R}{e} +1} + \frac{R}{r}} G \frac{r}{\sqrt{d}+2} 
            + \parenthese{2-\frac{1}{e}}G \frac{r}{\sqrt{d}+2} + 2\beta + C} T^{8/9} \nonumber \\
        &\quad + \frac{\beta D^2}{2}T^{1/9} 
                + \frac{3}{2} D \frac{d \parenthese{\sqrt{d}+2}}{r} P_{n,\lambda_2} T^{2/9} + \beta D^2 T^{3/9}
\end{align}
\end{proof}

\end{document}